%% file: qmc-trfmt.tex
\documentclass[11pt]{article}
\usepackage{times}
\usepackage{amsopn,amssymb,amsthm,amsmath}
\usepackage{paralist}
\usepackage{algpseudocode}
\usepackage{algorithm,algorithmicx}
\usepackage{color}
\usepackage{graphicx}
\usepackage{comment}
\usepackage{hyperref}
\usepackage{natbib}
\usepackage{multirow}
\usepackage{float}
\usepackage{subfigure}

\newif\ifwithcomments
\withcommentstrue
\ifwithcomments
\newcommand{\Vikas}[1]{{\color{red} Vikas: #1}}
\else
\newcommand{\Vikas}[1]{}
\fi
\ifwithcomments
\newcommand{\Jiyan}[1]{{\color{green} Jiyan: #1}}
\else
\newcommand{\Jiyan}[1]{}
\fi
\ifwithcomments
\newcommand{\Haim}[1]{{\color{blue} Haim: #1}}
\else
\newcommand{\Haim}[1]{}
\fi
\ifwithcomments
\newcommand{\Michael}[1]{{\color{cyan} Michael: #1}}
\else
\newcommand{\Michael}[1]{}
\fi

\DeclareMathOperator{\argmin}{arg\,min}
\DeclareMathOperator{\Real}{Re}
\DeclareMathOperator{\erf}{erf}
\DeclareMathOperator{\rect}{rect}
\DeclareMathOperator{\Vol}{Vol}
\DeclareMathOperator{\disr}{disr}
\DeclareMathOperator{\diag}{diag}
\DeclareMathOperator{\sinc}{sinc}

\newlength{\defbaselineskip}
\renewcommand{\algorithmicrequire}{\textbf{Input:}}
\renewcommand{\algorithmicensure}{\textbf{Output:}}

\begin{document}
\input{notation}

\title{Quasi-Monte Carlo Feature Maps for Shift-Invariant Kernels
\thanks{A short version of this paper has been presented in ICML 2014.}
}
       
\author{
       Haim Avron\thanks{Equal contributors} 
       \thanks{
       Mathematical Sciences \& Analytics,
       IBM T.J. Watson Research Center,
       Yorktown Heights, NY 10598, USA.
       Email: haimav@us.ibm.com }
       \and
       Vikas Sindhwani\footnotemark[2]
       \thanks{
        Google Research,
       New York, NY 10011, USA.
       Email: sindhwani@google.com }
       \and
       Jiyan Yang\footnotemark[2]
        \thanks{
        Institute for Computational and Mathematical Engineering,
       Stanford University,
       Stanford, CA 94305, USA.
       Email: jiyan@stanford.edu }
       \and
       Michael W. Mahoney 
       \thanks{
        International Computer Science Institute and Department of Statistics,
       University of California at Berkeley,
       Berkeley, CA 94720, USA.
       Email: mmahoney@stat.berkeley.edu }
       }

\date{}
\maketitle

\begin{abstract}%
We consider the problem of improving the efficiency of randomized Fourier
feature maps to accelerate training and testing speed of
kernel methods on large datasets. These approximate feature maps
arise as Monte Carlo approximations to integral representations of
shift-invariant kernel functions (e.g., Gaussian kernel). In this paper, we
propose to use {\it Quasi-Monte Carlo} (QMC) approximations instead,
where the relevant integrands are evaluated on a low-discrepancy sequence of
points as opposed to random point sets as in the Monte Carlo approach. We derive
a new discrepancy measure called {\it box discrepancy} based on theoretical
characterizations of the integration error with respect to a given sequence. We
then propose to learn QMC sequences adapted to our setting based on explicit box
discrepancy minimization. Our theoretical analyses are complemented with
empirical results that demonstrate the effectiveness of classical and adaptive QMC
techniques for this problem.
\end{abstract}

\section{Introduction}
\input{introduction}

\section{Preliminaries}\label{sec:pre}

\subsection{Notation}
We use $i$ both for subscript and for denoting $\sqrt{-1}$, relying on the context to distinguish between the two.
We use $y, z, \ldots$ to denote scalars. We use $\w, \t, \x \ldots$ to denote vectors,
and use $w_i$ to denote the $i$-th coordinate of vectors $\w$.
Furthermore, in a sequence of vectors, we use $\w_i$ to denote the $i$-th element of the sequence
and use $w_{ij}$ to denote the $j$-th coordinate of vector $\w_i$.
Given $\x_1, \ldots, \x_n$, the Gram matrix is defined as $\vv{K} \in \reals^{n \times n}$
where $\vv{K}_{ij} = k(\x_i, \x_j)$ for $i,j = 1,\ldots, n$.
We denote the error function by $\erf(\cdot)$, i.e.,  $\erf(z)=\int^z_0 e^{-z^2} dz$ for $z \in \complex$; see \citet{Wei94} and \citet{Mori83} for more details.

In ``\emph{MC sequence}'' we mean points drawn randomly either from the unit cube or certain
distribution that will be clear from the text.
For ``\emph{QMC sequence}'' we mean a deterministic sequence designed to reduce the
integration error. Typically, it will be a low-discrepancy sequence on the unit cube.

It is also useful to recall the definition of Reproducing Kernel Hilbert Space (RKHS).
\begin{definition}[Reproducing Kernel Hilbert Space~\citep{BerlinetAgnanBook}]
\label{def:rkhs}
A {\em reproducing kernel Hilbert space (RKHS)} is a Hilbert Space $\H:\X \to
\complex$ that possesses a reproducing kernel, i.e., a function $h : \X \times \X \to \complex
$ for which the following hold for all $\x \in \X$ and $f \in \H$:
\begin{itemize}
\item  $h(\vv{x}, \cdot) \in \H$
\item $\langle f, h(\vv{x}, \cdot )\rangle_{\H}  = f (\x)$ ({\em Reproducing Property})
\end{itemize}
\end{definition}

Equivalently, RKHSs are Hilbert spaces with bounded, continuous evaluation functionals. Informally, they are Hilbert spaces
with the nice property that if two functions $f,g\in \H$ are close in the sense of the distance derived from the norm in $\H$
(i.e., $\|f-g\|_{\H}$ is small), then their values $f(\x), g(\x)$ are also close for all $\x \in \X$; in other words, the norm
controls the pointwise behavior of functions in $\H$~\citep{BerlinetAgnanBook}.

\subsection{Related Work}\label{sec:related}
\input{relatedwork}

\subsection{Quasi-Monte Carlo Techniques: an Overview} \label{sec:qmc_review}
\input{qmc_overview}

\section{QMC Feature Maps: Our Algorithm} \label{sec:qmc_algos}
\input{qmc_algorithm}

\section{Theoretical Analysis and Average Case Error Bounds} \label{sec:qmc_theory}
\input{rkhs_bounds}

\section{Learning Adaptive QMC Sequences} \label{sec:qmc_adaptive}
\input{adaptive_qmc}

\section{Experiments} \label{sec:empirical}
\input{empirical}

\section{Conclusion and Future Work}
\input{conclusion}

\section*{Acknowledgements}
The authors would like to thank Josef Dick for useful pointers to literature about improvement of the QMC sequences;
Ha Quang Minh for several discussions on Paley-Wiener spaces and RKHS theory;
the anonymous ICML reviewers for pointing out the connection to herding and other
helpful comments; the anonymous JMLR reviewers for suggesting weighted sequences and other
helpful comments.
This research was supported by the XDATA program of the Defense Advanced Research Projects Agency (DARPA), administered through Air Force Research Laboratory contract FA8750-12-C-0323.
This work was done while J. Yang was a summer intern at IBM Research.

\appendix
\section{Technical Details} \label{proofs}
\input{supplementary2}

\bibliographystyle{plainnat}
\bibliography{qmc}

\end{document}

%% file: notation.tex
\newtheorem{theorem}{Theorem}
\newtheorem{conjecture}[theorem]{Conjecture}
\newtheorem{definition}[theorem]{Definition}
\newtheorem{lemma}[theorem]{Lemma}
\newtheorem{proposition}[theorem]{Proposition}
\newtheorem{corollary}[theorem]{Corollary}
\newtheorem{claim}[theorem]{Claim}
\newtheorem{fact}[theorem]{Fact}
\newtheorem{openprob}[theorem]{Open Problem}
\newtheorem{remark}[theorem]{Remark}
\newtheorem{exmp}[theorem]{Example}
\newtheorem{apdxlemma}{Lemma}

\def\X{\mathcal{X}}
\def\Y{\mathcal{Y}}
\def\P{\mathcal{P}}
\def\H{\mathcal{H}}
\def\E{\mathcal{E}}
\def\F{\mathcal{F}}
\def\N{\mathcal{N}}
\def\B{\mathbb{B}}

\newcommand{\ExpectSub}[2]{\mbox{}{\mathbb{E}}_{#1}\left[#2\right]}
\newcommand{\Expect}[1]{\mbox{}{\mathbb{E}}\left[#1\right]}

\def\reals{\mathbb{R}}
\def\complex{\mathbb{C}}
\def\Expectation{\mathbb{E}}
\newcommand*\conj[1]{\overline{#1}}

\newcommand{\vv}[1] {\mathbf{#1}}
\def\b{\vv{b}}
\def\u{\vv{u}}
\def\v{\vv{v}}
\def\t{\vv{t}}
\def\x{\vv{x}}
\def\y{\vv{y}}
\def\w{\vv{w}}
\def\z{\vv{z}}
\newcommand{\ww}[1] {\boldsymbol{#1}}
\def\wdelta{\ww{\delta}}
\def\womega{\ww{\omega}}
\def\wphi{\ww{\phi}}
\def\Dbox{D^{\Box}}
\def\Dboxb{D^{\Box \b}}
\newtheorem{defn}{Definition}
\newtheorem{thm}{Theorem}

%% file: introduction.tex
Kernel methods~\citep{KernelsBook,WahbaBook,CS01} offer a comprehensive suite of mathematically
 well-founded non-parametric modeling techniques for a wide range of 
 problems in machine learning. These include  nonlinear classification, regression, clustering, semi-supervised learning~\citep{MR}, time-series analysis~\citep{KernelsTimeSeries},
 sequence modeling~\citep{RKHSHMM}, dynamical systems~\citep{RKHSPSR}, 
 hypothesis testing~\citep{KernelsHypothesisTesting}, 
 causal modeling~\citep{KernelsCausality} and many more.

The central object of kernel methods is a kernel function $k: \X \times \X \to \reals$
 defined on an input domain $\X \subset \reals^d$ \footnote{In fact, $\X$ can be a rather
 general set. However, in this paper it is restricted to being a subset of $\reals^d$.}.
 The kernel $k$ is (non-uniquely) associated with an
embedding of the input space into a high-dimensional Hilbert space $\H$ (with
inner product $\langle \cdot, \cdot \rangle_{\H}$) via a feature map, $\Psi:
\X \to \H$, such that $$k(\vv{x}, \vv{z}) = \langle \Psi(\vv{x}),
\Psi(\vv{z})\rangle_{\H}~.$$ Standard regularized linear statistical models
in $\H$ then provide non-linear inference with respect to the original input
representation. The algorithmic basis of such constructions are classical
Representer Theorems~\citep{WahbaBook,KernelsBook} that guarantee
finite-dimensional solutions of associated optimization problems, even if $\H$ is infinite-dimensional.

However, there is a steep price of these elegant generalizations in terms of scalability. Consider,
for example, least squares regression given $n$ data points $\{(\vv{x}_i, y_i)\}_{i=1}^n$ and
assume that $n \gg d$. The complexity of linear regression training using
standard least squares solvers is $O(n d^2)$, with $O(nd)$ memory requirements,
and $O(d)$ prediction speed on a test point. Its kernel-based nonlinear
counterpart, however, requires solving a linear system involving the Gram matrix of the kernel
function (defined by $\vv{K}_{ij} = k(\vv{x}_i, \vv{x}_j)$). In general, this
incurs $O(n^3 + n^2d)$ complexity for training, $O(n^2)$ memory requirements,
and $O(nd)$ prediction time for a single test point -- none of which are
particularly appealing in ``big data'' settings. Similar conclusions apply to other algorithms
such as Kernel PCA.

This is rather unfortunate, since non-parametric models, such as the ones produced by kernel methods,
are particularly appealing in a big data settings as they can adapt to the full complexity of the underlying domain,
as uncovered by increasing dataset sizes. It is well-known that imposing strong structural constraints upfront
for the purpose of allowing an efficient solution (in the above example: a linear hypothesis space) often
limits, both theoretically and empirically, the potential to deliver value on large amounts of data.
Thus, as big data becomes pervasive across a number of  application domains,
it has become necessary to be able to develop highly scalable algorithms for kernel methods.

Recent years have seen intensive research on improving the scalability of kernel methods;
we review some recent progress in the next section.
In this paper, we revisit one of the most successful techniques, namely the randomized construction
of a family of low-dimensional approximate feature maps proposed by~\citet{RahimiRecht}.
These randomized feature maps, $\hat{\Psi}: \X \to \complex^s$, provide low-distortion approximations
for (complex-valued) kernel functions $k: \X \times \X \to \complex$:
\begin{equation}\label{eq:approx_featuremap}
k(\vv{x}, \vv{z}) \approx \langle \hat{\Psi}(\vv{x}),
\hat{\Psi}(\vv{z}) \rangle_{\complex^s}
\end{equation} where $\complex^s$ denotes the
space of $s$-dimensional complex numbers with the inner product, $\langle
\alpha, \beta\rangle_{\complex^s} = \sum_{i=1}^s \alpha_i \beta_i^*$, with
$z^*$ denoting the conjugate of the complex number $z$ (though~\citet{RahimiRecht} also
define real-valued feature maps for real-valued kernels, our technical exposition is
simplified by adopting the generality of complex-valued features).
The mapping $\hat{\Psi}(\cdot)$ is now applied to each of the
data points, to obtain a randomized feature representation of the data. We then apply a simple linear
method to these random features. That is, if our data is $\{(\vv{x}_i, y_i)\}_{i=1}^n$ we
learn on $\{(\z_i, y_i)\}_{i=1}^n$ where $\z_i = \hat{\Psi}(\x_i)$.
As long as $s$ is sufficiently smaller than $n$, this leads to more scalable solutions,
e.g., for regression we get back to $O(n s^2)$ training and $O(sd)$ prediction time, with
$O(ns)$ memory requirements. This technique is immensely successful, and has been used
in recent years to obtain state-of-the-art accuracies for some classical datasets~\citep{KernelTIMIT, DaiEtAl14, KernelsHilbertJSM, LuEtAl14}.

The starting point of~\citet{RahimiRecht}
is a celebrated result that characterizes the class of positive definite functions:
\begin{definition} A  function $g:\reals^d \mapsto \complex$ is a {\em positive definite function}
if for any set of $m$ points,
$\vv{x}_1\ldots \vv{x}_m \in \reals^d$, the $m\times m$ matrix $\vv{A}$ defined by $\vv{A}_{ij} = g(\vv{x}_i -
\vv{x}_j)$ is positive semi-definite.
\end{definition}
\begin{theorem}[\citet{Bochner1933}] A complex-valued function $g:\reals^d
\mapsto \complex$ is positive definite if and only if it is the Fourier
Transform of a finite non-negative Borel measure $\mu$ on $\reals^d$, i.e.,
$$
g(\vv{x}) = \hat{\mu}(\vv{x}) = \int_{\reals^d} e^{-i \vv{x}^T \vv{w}}
d\mu(\vv{w}), ~~~~\forall \vv{x} \in \reals^d~.
$$
\end{theorem}
Without loss of generality, we assume henceforth that $\mu(\cdot)$ is a
probability measure with associated probability density function $p(\cdot)$.

A kernel function $k:\reals^d \times \reals^d\mapsto \complex$ on $\reals^d$ is
called {\em shift-invariant} if $k(\vv{x}, \vv{z}) = g(\vv{x} - \vv{z})$,   for some
 positive definite function $g:\reals^d \mapsto \complex$.
Bochner's theorem implies that a scaled shift-invariant kernel can therefore be put
into one-to-one correspondence with a density $p(\cdot)$ such that,
\begin{equation}
k(\vv{x}, \vv{z}) = g(\vv{x} - \vv{z}) = \int_{\reals^d} e^{-i
(\vv{x}-\vv{z})^T
\vv{w}} p(\vv{w}) d\vv{w}~. \label{eq:bochner}
\end{equation}
For the most notable member of the shift-invariant family of kernels --
the Gaussian kernel: $$k(\vv{x}, \vv{z}) = e^{-\frac{\|\vv{x} -
\vv{z}\|^2_2}{2\sigma^2}},$$ the associated density is again Gaussian $\N(0,
\sigma^{-2}\vv{I}_d)$.

The integral representation of the kernel~\eqref{eq:bochner} may be
approximated as follows:
\begin{eqnarray*}
k(\vv{x}, \vv{z}) &=& \int_{\reals^d} e^{-i(\vv{x}-\vv{z})^T  \vv{w}} p(\vv{w}) d\vv{w}\label{eq:ourintegral} \\
&\approx &  \frac{1}{s}\sum_{j=1}^s e^{-i (\vv{x}-\vv{z})^T \vv{w}_s}
\label{eq:approx}\\
& = & \langle \hat{\Psi}_S(\vv{x}),
\hat{\Psi}_S(\vv{z}) \rangle_{\complex^s}~, \label{eq:innerproduct}
\end{eqnarray*} through the feature map,
\begin{equation}
\label{eq:features}
\hat{\Psi}_S(\vv{x}) = \frac{1}{\sqrt{s}}\left[ e^{-i \vv{x}^T \vv{w}_1}\ldots
e^{-i \vv{x}^T \vv{w}_s} \right] \in \complex^s~.
\end{equation}
The subscript $S$ denotes dependence of the feature map on the sequence
$S = \{\w_1,\ldots, \w_s\}$.

The goal of this work is to improve the convergence behavior of this approximation,
so that a smaller $s$ can be used to get the same quality of approximation to the
kernel function. This is motivated by recent work that showed that in order
to obtain state-of-the-art accuracy on some important datasets, a very large number of random features 
is needed~\citep{KernelTIMIT, KernelsHilbertJSM}. 

Our point of departure from the work
of~\citet{RahimiRecht} is the simple observation that when
$\w_1,\dots,\w_s$ are  are drawn from
the distribution defined by the density function $p(\cdot)$,
the approximation in~\eqref{eq:approx} may be viewed as a standard {\em
Monte Carlo} (MC) approximation to the integral representation of the kernel.
Instead of using plain MC approximation, we propose to use the low-discrepancy properties of
{\em Quasi-Monte Carlo} (QMC) sequences to reduce the integration error in approximations
of the form~\eqref{eq:approx}. A self-contained overview of Quasi-Monte Carlo
techniques for high-dimensional integration problems is provided in
Section~\ref{sec:pre}. In Section~\ref{sec:qmc_algos}, we describe how QMC techniques
apply to our setting.

We then proceed to apply an average case theoretical analysis of the
integration error for any given sequence $S$ (Section~\ref{sec:qmc_theory}). This bound motivates
an optimization problem over the sequence $S$ whose minimizer provides {\em adaptive QMC}
sequences fine tuned to our kernels (Section~\ref{sec:qmc_adaptive}).

Finally, empirical results (Section~\ref{sec:empirical}) clearly demonstrate the
superiority of QMC techniques over the MC feature maps~\citep{RahimiRecht}, the correctness of our
theoretical analysis and the potential value of adaptive QMC techniques for large-scale kernel methods.

%% file: relatedwork.tex
In this section we discuss related work on scalable kernel methods. Relevant work on QMC methods is discussed
in the next subsection.

Scalability has long been identified as a key challenge associated with deploying kernel methods in practice. One dominant line
of work constructs low-rank approximations of the Gram matrix, either using data-oblivious randomized feature maps to
approximate the kernel function, or using sampling techniques such as the classical Nystr\"{o}m method~\citep{WG01}. In its
vanilla version, the latter approach - Nystr\"{o}m method - samples points from the dataset, computes the columns of the Gram matrix that
corresponds to the sampled data points, and uses this partial computation of the Gram matrix to construct an approximation to
the entire Gram matrix. More elaborate techniques exist, both randomized and deterministic; see~\citet{GM13} for a thorough
treatment.

More relevant to our work is the randomized feature mapping approach. Pioneered by the seminal paper
of~\citet{RahimiRecht}, the core idea is to construct, for a given kernel on a data domain $\X$, a transformation  $\hat{\Psi}: \X \to \complex^s$ such that $k(\x, \z) \approx \langle \hat{\Psi}(\x), \hat{\Psi}(\z)
\rangle_{\complex^s}$. Invoking Bochner's theorem, a classical result in harmonic analysis, Rahimi and Recht show how to
construct a randomized feature map for shift-invariant kernels, i.e., kernels that can be written $k(\x,\z) = g(\x - \y)$ for
some positive definite function $g(\cdot)$.

Subsequently, there has been considerable effort given to extending this technique to other classes of
kernels. 
\citet{LIS10} use Bochner's theorem to provide random features to the wider class of
group-invariant kernels. \citet{MB09} suggested random features for the intersection kernel
$k(\x, \z) = \sum^d_{i=1}\min(x_i, z_i)$.
\citet{VZ11} developed feature maps for $\gamma$-homogeneous kernels. \citet{SVJZ10}
developed feature maps for generalized RBF kernels $k(\x,\z)=g(D(\x,\z)^2)$ where $g(\cdot)$ is a positive
definite function, and $D(\cdot,\cdot)$ is a distance metric. \citet{KK12} suggested feature maps
for dot-product kernels. The feature maps are based on the Maclaurin expansion, which is guaranteed
to be non-negative due to a classical result of~\citet{Sch42}. \citet{PP13} suggested feature maps
for the polynomial kernels. Their construction leverages known techniques from sketching theory. It can also
be shown that their feature map is an {\em oblivious subspace embedding}, and this observation provides stronger
theoretical guarantees than point-wise error bounds prevalent in the feature map literature~\citep{AHW14}.
By invoking a variant of Bochner's theorem that replaces the Fourier transform with the Laplace
transform,~\citet{YSFAM14} obtained randomized feature maps for semigroup kernels on histograms.
We note that while the original feature maps suggested by Rahimi and Recht were randomized,
some of the aforementioned maps are deterministic.

Our work is more in-line with recent efforts on scaling up the random features, so that learning and
prediction can be done faster. \citet{LSS13} return to the original construction of~\citet{RahimiRecht},
and devise a clever distribution of random samples $\w_1, \w_2, \dots, \w_s$ that is structured so that
the generation of random features can be done much faster. They showed that only a very limited concession
in term of convergence rate is made. \citet{HGXD14}, working on the polynomial kernel,
suggest first generating a very large amount of random features,
and then applying them a low-distortion embedding based the Fast Johnson-Lindenstruass Transform, so
the make the final size of the mapped vector rather small. In contrast, our work tries to design
$\w_1, \dots, \w_s$ so that less features will be necessary to get the same quality of kernel approximation.

Several other scalable approaches for large-scale kernel methods have been suggested over the years, starting from approaches
such as chunking and decomposition methods proposed in the early days of SVM optimization literature. \citet{RD06}
use an improved fast Gauss transform for large scale Gaussian Process regression. There are also approaches that are more
specific to the objective function at hand, e.g., \citet{KCD06} builds a kernel expansion greedily to optimize the SVM
objective function. Another well known approach is the Core Vector Machines~\citep{CVM05} which draws on approximation
algorithms from computational geometry to scale up a class of kernel methods that can be reformulated in terms of the minimum
enclosing ball problem.

For a broader discussion of these methods, and others, see~\citet{KernelScalabilityBook}.

%% file: qmc_overview.tex
In this section we provide a self-contained overview of Quasi-Monte Carlo (QMC) techniques. For brevity,
we restrict our discussion to background that is necessary for understanding subsequent sections.
We refer the interested reader to the excellent reviews by~\citet{Caflisch98} and~\citet{DKS13}, and the recent book~\citet{QMCBook}
for a much more detailed exposition.

Consider the task of computing an approximation of the following integral
\begin{equation}
\label{eq:integral}
I_d[f] = \int_{[0,1]^d} f(\x)d\x~.
\end{equation}
One can observe that if $\x$ is a random vector uniformly distributed over $[0,1]^d$
then $I_d[f]=\Expect{f(\x)}$. An empirical approximation to the expected value
can be computed by drawing a random point set $S = \{\vv{w}_1,\ldots, \vv{w}_s\}$ independently from $[0,1]^d$, and computing:
$$
I_S[f]=\frac{1}{s}\sum_{\w \in S}f(\w)~.
$$
This is the Monte Carlo (MC) method.

Define the integration error with respect to the point set $S$ as,
\begin{equation*}\epsilon_S[f]=\lvert I_d(f) - I_S(f) \rvert~. \end{equation*}
When $S$ is drawn randomly, the Central Limit Theorem asserts that if $s=\lvert
S \rvert$ is large enough then $\epsilon_S[f] \approx \sigma[f] s^{-1/2}
\ww{\nu}$ where $\ww{\nu}$ is a standard normal random variable, and $\sigma[f]$
is the square-root of the variance of $f$; that is,
\begin{equation*}
\sigma^2[f] = \int_{[0,1]^d} \left( f(\x) - I_d(f) \right)^2 d\x~.
\end{equation*} In other words, the root mean square error of the Monte Carlo
method is,
\begin{equation}
\left(\ExpectSub{S}{\epsilon_S[f]^2}\right)^{1/2} \approx \sigma[f] s^{-1/2}.
\label{eq:MC_error}
\end{equation}
Therefore, the Monte Carlo method converges at a rate of $O(s^{-1/2})$.

The aim of QMC methods is to improve the convergence rate by using a
deterministic {\it low-discrepancy sequence} to construct $S$, instead of
randomly sampling points. The underlying intuition is illustrated in
Figure~\ref{fig:qmc_mc_seq}, where we plot a set of 1000 two-dimensional random points (left graph), and a set of 1000
two-dimensional points from a quasi-random sequence (Halton sequence; right
graph).
\begin{figure}[t]
  \centering \includegraphics[width=0.75\textwidth] {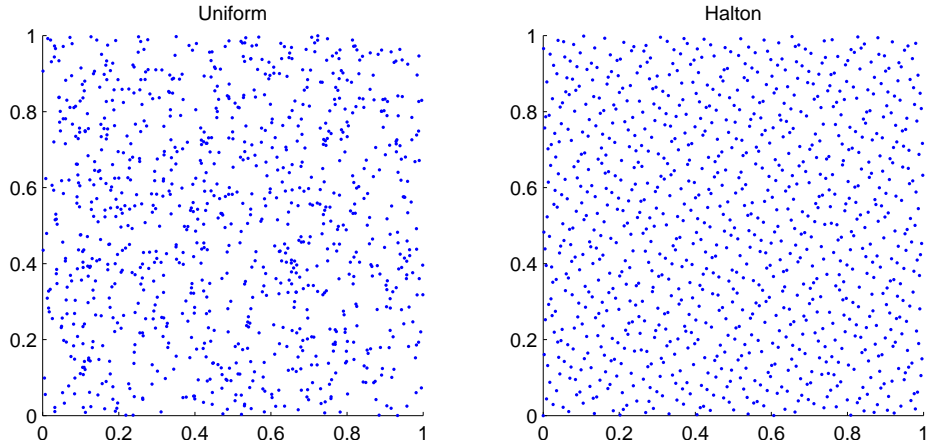}
  \caption{Comparison of MC and QMC sequences.}
   \label{fig:qmc_mc_seq}
\end{figure}
In the random sequence we see that there is an undesired clustering of points,
and as a consequence empty spaces. Clusters add little to the approximation of the
integral in those regions, while in the empty spaces the integrand is not
sampled. This lack of uniformity is due to the fact that Monte Carlo samples
are independent of each other. By carefully designing a sequence of
correlated points to avoid such clustering effects, QMC attempts to avoid this
phenomena, and thus provide faster convergence to the integral.

The theoretical apparatus for designing such sequences are inequalities of the
form $$
\epsilon_S(f) \leq D(S)V(f)~,
$$
in which $V(f)$ is a measure of the variation or difficulty of integrating $f(\cdot)$
and $D(S)$ is a sequence-dependent term that typically measures the {\em discrepancy}, or
degree of deviation from uniformity, of the sequence $S$. For example, the expected Monte Carlo
integration error decouples into a variance term, and $s^{-1/2}$ as
in~\eqref{eq:MC_error}.

A prototypical inequality of this sort is the following remarkable and classical result:
\begin{theorem}[Koksma-Hlawka inequality] \label{thm:kh} For any function $f$
with bounded variation, and sequence $S=\{\w_1, \ldots, \w_s\}$, the integration error is bounded
above as follows,
\begin{equation*} \epsilon_S[f] \leq D^\star(S)V_{HK}[f]~,
\end{equation*}
where $V_{HK}$ is the {\em Hardy-Krause variation of $f$}
(see~\citet{NiederreiterBook}), which is defined in terms of the following partial
derivatives,
\begin{equation}
V_{HK}[f] = \sum_{I\subset [d],I\neq \emptyset} \int_{[0,1]^{|I|}}\left|
\frac{\partial f}{\partial \vv{u}_I}\biggr\rvert_{u_j=1, j\notin I} \right|
d\vv{u}_I~,
\label{eq:Vhk}
\end{equation}
 and $D^\star$ is the {\em star
discrepancy} defined by
\begin{equation*}
\label{eq:stardisc}
D^\star(S)=\sup_{\x \in [0,1]^d}\lvert \disr_S (\x)\rvert~,
\end{equation*}
where $\disr_S$ is the {\em local discrepancy function}
$$
\disr_S(\x)=\Vol(J_{\x}) - \frac{\lvert \{ i~:~\w_i \in J_\x \}
\rvert}{s} $$
with $J_\x = [0,x_1) \times [0,x_2) \times \dots \times [0, x_d)$ with
$\Vol(J_\x)=\prod_{j=1}^d x_j$.
\end{theorem}

Given $\x$, the second term in $\disr_S(\x)$ is an estimate of the volume of $J_\x$, which will be accurate if the points in $S$ are uniform enough.
 $D^\star(S)$ measures the maximum difference between the actual volume of the subregion $J_\x$ and its estimate for all $\x$ in $[0,1]^d$.

An infinite sequence $\w_1, \w_2, \dots$ is defined to be a {\em low-discrepancy
sequence} if, as a function of $s$, $D^\star(\{ \w_1, \dots, \w_s\}) = O((\log
s)^d / s)$. Several constructions are know to be low-discrepancy sequences. One notable example is
the {\em Halton sequences}, which are defined as follows.  Let $p_1,\dots,p_d$ be the first $d$
prime numbers. The Halton sequence $\w_1, \w_2, \dots $ of dimension $d$ is defined by
$$
\w_i = (\phi_{p_1}(i),\dots,\phi_{p_d}(i))
$$
where for integers $i\geq0$ and $b\geq2$ we have
$$
\phi_b(i) = \sum^\infty_{a=1} i_a b^{-a}
$$
in which $i_0,i_1, \dots \in \{0,1,\dots,b-1\}$ is given by the unique decomposition
$$
i = \sum^\infty_{a=1} i_a b^{a - 1}~.
$$

It is outside the scope of this paper to describe all these constructions in detail. However
we mention that in addition to the Halton sequences, other notable members are {\em Sobol'
sequences}, {\em Faure sequences}, {\em Niederreiter sequences}, and more (see~\citet{DKS13}, Section $2$).
We also mention that it is conjectured that the $O((\log s)^d / s)$ rate for star discrepancy decay is optimal.

The classical QMC theory, which is based on the Koksma-Hlawka inequality and
low discrepancy sequences, thus achieves a convergence rate of $O((\log s)^d / s)$. While
this is asymptotically superior to $O(s^{-1/2})$ for a fixed $d$, it requires
$s$ to be exponential in $d$ for the improvement to manifest. As such, in the past QMC
methods were dismissed as unsuitable for very high-dimensional integration.

However, several authors noticed that QMC methods perform better than MC even for very high-dimensional
integration~\citep{QMCEfficiency, DKS13}.\footnote{Also see: ``On the unreasonable effectiveness of QMC", I.H.
Sloan~\url{https://mcqmc.mimuw.edu.pl/Presentations/sloan.pdf}} Contemporary QMC literature explains and expands on these
empirical observations, by leveraging the structure of the space in which the integrand function lives, to derive more
refined bounds and discrepancy measures, even when classical measures of variation such as~\eqref{eq:Vhk} are unbounded. This
literature has evolved along at least two directions: one, where worst-case analysis is provided under the assumption that
the integrands live in a Reproducing Kernel Hilbert Space (RKHS) of sufficiently smooth and well-behaved functions
(see~\citet{DKS13}, Section $3$) and second, where the analysis is done in terms of {\it average-case} error, under an
assumed probability distribution over the integrands, instead of worst-case
error~\citep{Wozniakowski,BreakingIntractability}. We refrain from more details, as these are essentially the paths that the
analysis in Section~\ref{sec:qmc_theory} follows for our specific setting. 

%% file: qmc_algorithm.tex

We assume that the density function in \eqref{eq:bochner} can be written as
$p(\x) = \prod_{j=1}^d p_j(x_j)$, where $p_j(\cdot)$ is a univariate density function.
The density functions associated to many shift-invariant kernels, e.g., Gaussian, Laplacian and Cauchy, admits such a form.

The QMC method is generally applicable to integrals over a unit cube. So typically
 integrals of the form~\eqref{eq:bochner} are handled by first generating a low
 discrepancy sequence $\t_1,\dots,\t_s \in [0,1]^d$, and
 transforming it into a sequence $\w_1,\dots,\w_s$ in $\reals^d$, instead of
 drawing the elements of the sequence from $p(\cdot)$ as in the MC method.


 To convert~\eqref{eq:bochner} to an integral over the unit cube, a simple change of variables suffices.
 For $\vv{t} \in \reals^d$, define
 \begin{equation}
 \Phi^{-1}(\t) = \left( \Phi_1^{-1}(t_1), \ldots,  \Phi_d^{-1}(t_d) \right) \in \reals^d~, \label{eq:cdf}
 \end{equation}
 where $\Phi_j(\cdot)$ is the cumulative distribution function (CDF) of $p_j(\cdot)$, for $j = 1,\ldots, d$.
 By setting $\vv{w} = \Phi^{-1}(\vv{t})$,
 then~\eqref{eq:bochner} can be equivalently written as
 \begin{equation*}
   \int_{\reals^d} e^{- i (\vv{x} - \vv{z})^T \vv{w}} p(\vv{w}) d \vv{w}
   = \int_{[0,1]^d} e^{-i (\vv{x} - \vv{z})^T\Phi^{-1}(\vv{t})} d
   \vv{t}~.
 \end{equation*}

 Thus, a low discrepancy sequence $\t_1,\dots,\t_s \in [0,1]^d$ can be transformed using
 $\w_i = \Phi^{-1}(\vv{\t}_i)$, which is then plugged into~\eqref{eq:features}
 to yield the QMC feature map. This simple procedure is summarized in
 Algorithm~\ref{alg:qmc}. QMC feature maps are analyzed in the next section.

 \begin{algorithm}[tb]
 \caption{Quasi-Random Fourier Features}
  \label{alg:qmc}
  \begin{algorithmic}[1]
    \Require  Shift-invariant kernel $k$, size $s$.

    \Ensure Feature map $\hat \Psi(\vv{x}): \reals^d \mapsto \complex^s$.

    \State Find $p$, the inverse Fourier transform of $k$.

    \State Generate a low discrepancy sequence $\t_1,\dots,\t_s$.

    \State Transform the sequence: $\w_i =  \Phi^{-1}(\vv{t}_i)$ by \eqref{eq:cdf}.

     \State  Set $\hat \Psi(\vv{x}) = \sqrt{\frac{1}{s}} \left[ e^{-i \x^T \w_1}, \ldots, e^{-i \x^T \w_s} \right] $.

    \end{algorithmic}
\end{algorithm}



%% file: rkhs_bounds.tex
The proofs for assertions made in this section and the next can be found in the Appendix.

The goal of this section is to develop a framework for analyzing
the approximation quality of the QMC feature maps described in the previous section
(Algorithm~\ref{alg:qmc}). We need to develop such a framework since the
classical Koksma-Hlawka inequality cannot be applied
to our setting, as the following proposition shows:
\begin{proposition}
\label{prop:unbounded-vhk}
For any $p(\x) = \prod_{j=1}^d p_j(x_j)$,
where $p_j(\cdot)$ is a univariate density function,
let
$$
 \Phi^{-1}(\t) = \left( \Phi_1^{-1}(t_1), \ldots,  \Phi_d^{-1}(t_d) \right)~.
$$
For a fixed $\u \in \reals^d$, consider
$f_{\vv{u}} (\vv{t}) = e^{-i \vv{u}^T
\Phi^{-1}(\vv{t}) }$, $\vv{t}\in [0,1]^d$.
The Hardy-Krause variation of $f_{\u}(\cdot)$ is unbounded. That is,
one of the integrals in the sum~\eqref{eq:Vhk} is unbounded.
\end{proposition}

Our framework is based on a new discrepancy measure, {\em box discrepancy}, that
characterizes integration error for the set of integrals defined with respect to
the underlying data domain. Throughout this section we use the convention that $S = \{\w_1,\ldots, \w_s\}$,
 and the notation $\bar{\X} = \left\{\x - \z~|~\x, \z \in \X\right\}$.

Given a probability density function $p(\cdot)$ and $S$, we define the integration error
$\epsilon_{S,p}[f]$ of a function $f(\cdot)$ with respect to $p(\cdot)$ and the $s$ samples as,
\begin{equation*}
\epsilon_{S,p}[f] = \left| \int_{\reals^d} f(\vv{x}) p(\vv{x}) d\vv{x} -
\frac{1}{s} \sum_{i=1}^s f(\vv{w}_i) \right|~.
\end{equation*}
We are interested in characterizing the behavior of $\epsilon_{S,p}[f]$ on $f\in\F_{\bar{\X}}$ where
\begin{equation*}
\F_{\bar{\X}} = \left\{ f_{\vv{u}}(\vv{x}) = e^{-i \vv{u}^T \vv{x}},~
\vv{u}\in \bar{\X} \right\}~.
\end{equation*}

As is common in modern QMC analysis~\citep{QMCBook,DKS13}, our analysis is based on setting up a Reproducing Kernel Hilbert
Space of ``nice" functions that is related to integrands that we are interested in, and using properties of the RKHS to
derive bounds on the integration error. In particular, the integration error of integrands in an RKHS can be bounded using the
following proposition.

\begin{proposition}[Integration Error in an RKHS] \label{prop:int-err-rkhs} Let $\H$ be an RKHS with kernel $h(\cdot, \cdot)$.
Assume that $ \kappa = \sup_{\x \in \reals^d} h(\x,\x) < \infty$. Then, for all $f\in \H$ we have, \begin{equation}
\label{eq:error_f} \epsilon_{S,p}[f] \leq \Vert f \Vert_\H D_{h,p}(S)~, \end{equation} where \begin{eqnarray} D_{h,p}(S)^2
&=& \left\|\int_{\reals^d} h(\omega, \cdot) p(\omega) d\omega - \frac{1}{s} \sum_{l=1}^s h(\w_l, \cdot)\right\|^2_{\H}
\label{eq:worstcase_error}\\ & = &
 \int _{\reals^d}\int_{\reals^d} h(\omega, \phi) p(\omega)
p(\phi) d\omega d\phi \nonumber - \frac{2}{s}\sum^s_{l=1} \int_{\reals^d}
h(\w_l, \omega) p(\omega) d\omega  \nonumber \\
& & \qquad + \frac{1}{s^2}\sum_{l=1}^s \sum_{j=1}^s
h(\w_l,\w_j)~.\nonumber
\end{eqnarray}
\end{proposition}
 \begin{remark}

In the theory of RKHS embeddings of probability distributions~\citep{SmolaEmbeddings07,Sriperumbudur10}, the function
$$
\mu_{h,p}(\x)= \int_{\reals^d} h(\omega, \x) p(\omega) d\omega
$$
is known as the {\em kernel mean} embedding of $p(\cdot)$. The function
$$
\hat{\mu}_{h,p,S}(\x)= \frac{1}{s} \sum_{l=1}^s h(\w_l, \x)
$$
is then the empirical mean map.
 \end{remark}

The RKHS we use is as follows. For a vector $\b \in \reals^d$, let us define $\Box \b =
\{\u \in \reals^d ~|~|u_j|\leq b_j\}$. Let
\begin{equation*}
\F_{\Box \b} = \left\{ f_{\vv{u}}(\vv{x}) = e^{-i \vv{u}^T \vv{x}},~
\vv{u}\in \Box \b \right\}~, 
\end{equation*}
and consider the space of functions that admit an integral representation over
$\F_{\Box \b}$ of the form \begin{equation} f(\vv{x}) = \int_{\vv{u}\in \Box \b}
\hat{f}(\vv{u}) e^{-i \vv{u}^T \vv{x}} d\vv{u}~\textrm{where}~\hat{f}(\vv{u})\in L_2(\Box \b)~. \label{eq:bandlimited}
\end{equation}
This space is associated with bandlimited functions, i.e., functions with
compactly-supported inverse Fourier transforms, which are of fundamental importance in
the Shannon-Nyquist sampling theory. Under a natural choice of inner product, these
spaces are called {\em Paley-Wiener spaces} and they constitute an RKHS.
\begin{proposition}[Kernel of Paley-Wiener 
RKHS~\citep{BerlinetAgnanBook,Yao67,holomorphic}] By $PW_{\b}$, denote the space
of functions which admit the representation in~\eqref{eq:bandlimited}, with the inner product $\langle f, g \rangle_{PW_{\b}} =
(2\pi)^{2d}\langle \hat{f}, \hat{g} \rangle_{L_2(\Box \b)}$. $PW_{\b}$ is an RKHS with kernel
function,
\begin{equation*}
\sinc_{\vv{b}}(\vv{u}, \vv{v}) = \pi^{-d} \prod_{j=1}^d \frac{\sin \left(b_j (u_j -
v_j)\right)}{u_j - v_j}~.
\end{equation*}
For notational convenience, in the above we define $\sin(b\cdot 0)/0$ to be $b$.
Furthermore, $\langle f, g \rangle_{PW_{\b}} = \langle f, g \rangle_{L_2(\Box \b)}$.
\end{proposition}

If $b_j = \sup_{\u \in \bar{\X}}|u_j|$ then $\bar{\X} \subset \Box \b$, so $F_{\bar{\X}} \subset F_{\Box \b}$. Since we
wish to bound the integration error on functions in $F_{\bar{\X}}$, it suffices to bound the integration error on $\F_{\Box
\b}$. Unfortunately, while $\F_{\Box \b}$ defines $PW_\b$, the functions in it, being not square integrable, are {\em not}
members of $PW_\b$, so analyzing the integration error in $PW_\b$ do not directly apply to them. 
However,  damped approximations of $f_{\vv{u}}(\cdot)$ of the form
$\tilde{f}_{\vv{u}} (\vv{x}) = e^{-i \vv{u}^T \vv{x}} \sinc (T\vv{x})$ are
 members of $PW_{\vv{b}}$ with $\|\tilde{f}\|_{PW_\b} = \frac{1}{\sqrt{T}}$.
Hence, we expect the analysis of the integration error in $PW_\b$ to provide
provide a discrepancy measure for integrating
functions in $\F_{\Box \b}$.

For $PW_{\b}$
the discrepancy measure $D_{h,S}$ in Proposition~\ref{prop:int-err-rkhs}
can be written explicitly.
\begin{theorem}[Discrepancy in $PW_{\b}$]
\label{thm:discr-pw}
Suppose that $p(\cdot)$ is a probability density function, and that we can write
$p(\x) = \prod^d_{j=1} p_j(x_j)$ where each $p_j(\cdot)$ is a univariate probability density
function as well. Let $\varphi_j(\cdot)$ be the characteristic function associated
with $p_j(\cdot)$.
Then,
\begin{eqnarray}
D_{\sinc_{\vv{b}}, p}(S)^2 &=& \pi^{-d} \prod_{j=1}^d \int_{-b_j}^{b_j} \lvert
\varphi_j(\beta) \rvert^2 d\beta - \nonumber \\
& & \frac{2(2\pi)^{-d}}{s} \sum_{l=1}^s
\prod_{j=1}^d \int_{-b_j}^{b_j}\varphi_j(\beta)
e^{i w_{lj}\beta}d\beta + \nonumber\\
& & \frac{1}{s^2}\sum_{l=1}^s
\sum_{j=1}^s \sinc_\b(\w_l,\w_j)~.
\label{eq:discr-sinc}
\end{eqnarray}
\end{theorem}

This naturally leads to the definition of the {\em box discrepancy},
analogous to the star discrepancy described in Theorem~\ref{thm:kh}.
\begin{definition}[Box Discrepancy]
The box discrepancy of a sequence $S$ with respect to $p(\cdot)$ is defined as,
$$
\Dboxb_p(S) = D_{\sinc_{\vv{b}}, p} (S)~.
$$
\end{definition}
For notational convenience, we generally omit the $\b$ from $\Dboxb_p(S)$
as long as it is clear from the context.

The worse-case integration error bound for Paley-Wiener spaces is stated in the following
as a corollary of Proposition~\ref{prop:int-err-rkhs}. As explained earlier, 
this result not yet apply to functions in $\F_{\Box \b}$ because these functions
are not part of $PW_\b$. Nevertheless, we state it here for completeness.
\begin{corollary}[Integration Error in $PW_{\b}$]
For $f \in PW_{\b}$ we have
$$\epsilon_{S,p}[f] \leq \Vert f \Vert_{PW_\b} \Dbox_p(S)~.$$
\end{corollary}

Our main result shows that the expected square error of an
integrand drawn from a uniform distribution over $\F_{\Box \b}$ is proportional to the
square discrepancy measure $\Dbox_p (S)$. This result is in the spirit of similar
average case analysis in the QMC literature~\citep{Wozniakowski,BreakingIntractability}.
\begin{theorem}[Average Case Error]
\label{thm:average-case}
Let ${\cal U}(\F_{\Box \b})$ denote the uniform distribution on $\F_{\Box \b}$. That is,
$f\sim {\cal U}(\F_{\Box \b})$ denotes $f = f_\u$ where $f_\u(\vv{x}) = e^{-i \vv{u}^T \vv{x}}$ and $\u$ is randomly drawn from a uniform distribution on $\Box \b$.
We have,
\begin{equation*}
\ExpectSub{f\sim {\cal U}(\F_{\Box \b})}{\epsilon_{S,
p} [f]^2} = \frac{\pi^{d}}{\prod^d_{j=1} b_j} \Dbox_p(S)^2~.
\end{equation*}
\end{theorem}

We now give an explicit formula for $\Dbox_p(S)$ for the case that $p(\cdot)$ is the density function
of the multivariate Gaussian distribution with zero mean and independent
components. This is an important special case since this is
the density function that is relevant for the Gaussian kernel.
\begin{corollary}[Discrepancy for Gaussian Distribution]
\label{cor:discr-gaussian}
Let $p(\cdot)$ be the $d$-dimensional multivariate Gaussian density function with zero mean
and covariance matrix equal to $\diag(\sigma_1^{-2}, \dots, \sigma_d^{-2})$.
We have,
\begin{eqnarray}
 \Dbox_p(S)^2 &=& \frac{1}{s^2}\sum_{l=1}^s
\sum_{j=1}^s \sinc_\b(\w_l,\w_j)  - \nonumber \\
& & \frac{2}{s} \sum_{l=1}^s
 \prod_{j=1}^d c_{lj}
\Real\left(\erf\left(\frac{b_j}{\sigma_j\sqrt{2}} - i\frac{\sigma_j
w_{lj}}{\sqrt{2}}\right)\right) +\nonumber \\
& &  + \prod_{j=1}^d \frac{\sigma_j}{2\sqrt{\pi}} \erf\left(\frac{b_j}{\sigma_j}\right)~, \label{eq:discrepancy_gaussian}
\end{eqnarray}
where $$c_{lj} = \left(\frac{\sigma_j}{\sqrt{2\pi}}\right)e^{-\frac{\sigma_j^2 w^2_{lj}}{2}}~.$$
\end{corollary}

Intuitively, the box discrepancy of the Gaussian kernel can be interpreted as follows. The function $\sinc(x) = \sin(x)/x$
achieves its maximum at $x = 0$ and minimizes at discrete values of $x$ decaying to 0 as $|x|$ goes to $\infty$. Hence the
first term in~\eqref{eq:discrepancy_gaussian} tends to be minimized when the pairwise distance between $\w_j$ are
sufficiently separated. Due to the shape of cumulative distribution function of Gaussian distribution, the values
of $\t_j=\Phi(\w_j)~(j=1,\dots,s)$ are driven to be close to the boundary of the unit cube.
As for second term, the original expression is -$\frac{2}{s}\sum^s_{l=1}
\int_{\reals^d} h(\w_l, \omega) p(\omega) d\omega$. This term encourages the sequence $\{\w_l\}$ to mimic samples from
$p(\omega)$. Since $p(\omega)$ concentrates its mass around $\omega = 0$,  the $\w_j$ also concentrates around $0$ to
maximize the integral and therefore the values of $\t_j=\Phi(\w_j)~(j=1,\dots,s)$ are driven closer to the center of the unit cube. Sequences with low box discrepancy
therefore optimize a tradeoff between these competing terms.

Two other shift-invariant kernel that have been mentioned in the machine learning literature is the Laplacian kernel~\citep{RahimiRecht} and Matern kernel~\citep{LSS13}. The distribution associated with the Laplacian kernel can be written as a product $p(\x) = \prod_{j=1}^d p_j(x_j)$, where $p_j(\cdot)$ is density associated with the Cauchy distribution. The characteristic function is simple ($\phi_j(\beta) = e^{-\left| \beta \right|/\sigma_j}$) so analytic formulas like~\eqref{eq:discrepancy_gaussian} can be derived. The distribution associated with the Matern kernel, on the other hand, is the multivariate t-distribution, which cannot be written as a product $p(\x) = \prod_{j=1}^d p_j(x_j)$, so the presented theory does not apply to it.

\paragraph{Discrepancy of Monte-Carlo Sequences.\\}
We now derive an expression for the expected discrepancy of Monte-Carlo sequences, and show that it decays as
$O(s^{-1/2})$. This is useful since via an averaging argument we are guaranteed that there exists sets
for which the discrepancy behaves $O(s^{-1/2})$.
\begin{corollary}
\label{cor:disc_mc_general}
Suppose $\t_1,\ldots,\t_s$ are chosen uniformly from $[0,1]^d$. Let $\w_i = \Phi^{-1}(\t_i)$, for $i = 1,\ldots,s$. Assume  that $ \kappa = \sup_{\x \in \reals^d} h(\x,\x) < \infty$.
Then
 \begin{equation*}
   \Expect{D_{h,p}(S)^2} = \frac{1}{s} \int _{\reals^d} h(\omega, \omega) p(\omega) d\omega - \frac{1}{s} \int _{\reals^d}\int_{\reals^d} h(\omega, \phi) p(\omega)
p(\phi) d\omega d\phi~.
 \end{equation*}
\end{corollary}

Again, we can derive specific formulas for the Gaussian density. The following is straightforward from Corollary~\ref{cor:disc_mc_general}. We omit the proof.
\begin{corollary}
Let $p(\cdot)$ be the $d$-dimensional multivariate Gaussian density function with zero mean
and covariance matrix equal to $\diag(\sigma_1^{-2}, \dots, \sigma_d^{-2})$.
Suppose $\t_1,\ldots,\t_s$ are chosen uniformly from $[0,1]^d$. Let $\w_i = \Phi^{-1}(\t_i)$, for $i = 1,\ldots,s$.
Then,
 \begin{equation}
  \label{eq:expec_mc_disc}
   \Expect{\Dbox_p(S)^2} = \frac{1}{s}\left( \pi^{-d} \prod_{j=1}^d b_j - \prod_{j=1}^d \frac{\sigma_j}{2\sqrt{\pi}} \erf\left(\frac{b_j}{\sigma_j}\right) \right) .
 \end{equation}
\end{corollary}



%% file: adaptive_qmc.tex
For simplicity, in this section we assume that $p(\cdot)$ is the density function of Gaussian distribution with zero mean.
We also omit the subscript $p$ from $\Dbox_p$.
Similar analysis and equations can be derived for other density functions.

Error characterization via discrepancy measures like~\eqref{eq:discrepancy_gaussian} is typically used in the QMC literature to prescribe sequences whose discrepancy behaves favorably. It is clear that for the box discrepancy, a meticulous design is needed for a high quality sequence and we leave this to future work. Instead, in this work, we use the fact that
 unlike the star discrepancy~\eqref{eq:stardisc}, the box discrepancy is a smooth function of the sequence
with a closed-form formula. This allows us to both evaluate various candidate
sequences, and select the one with the lowest discrepancy, as well as to
{\em adaptively learn} a QMC sequence that is specialized for our problem
setting via numerical optimization. The basis is the following proposition, which gives
an expression for the gradient of $\Dbox(S)$.
\begin{proposition}[Gradient of Box Discrepancy]
\label{prop:gradient}
Define the following scalar functions and variables,
\begin{eqnarray}
\sinc'(z) &=& \frac{\cos(z)}{z} - \frac{\sin(z)}{z^2},~~\sinc'_b(z) = \frac{b}{\pi}\sinc'(bz)~;
\nonumber
\\
c_j &=& \left(\frac{\sigma_j }{\sqrt{2\pi}}\right),  j=1,\ldots, d~;
\nonumber \\
g_j(x) &=& c_j e^{-\frac{\sigma_j^2}{2} x^2} \Real
\left(\erf\left[\frac{b_j}{\sigma_j\sqrt{2}} - i\frac{\sigma_j
x}{\sqrt{2}} \right]\right)~; \nonumber\\
g_j'(x)&=& - \sigma_j^2 x g_j(x)
              + \sqrt{\frac{2}{\pi}} c_j \sigma_j  e^{-\frac{b_j^2}{2\sigma_j^2}}
              \sin(b_j x)~.
              \nonumber
\end{eqnarray}
In the above ,we define $\sinc'(0)$ to be $0$.
Then, the elements of the gradient vector of $\Dbox$ are given by,
\begin{eqnarray}
\frac{\partial \Dbox}{\partial w_{lj}}&=&
\frac{2}{s^2} \sum^{s}_{\substack{m=1\\ m\neq l}} \left( b_{j} \sinc'_{b_j} (w_{lj},
w_{mj})\prod_{q\neq j} \sinc_{b_q} (w_{lq}, w_{mq})\right)- \nonumber \\
& &
\frac{2}{s} g'_j(w_{lj}) \left(\prod_{q\neq j} g_q(w_{lq})\right)~.  \label{eq:gradient}
\end{eqnarray}
\end{proposition}

We explore two possible approaches for finding sequences based on optimizing the box discrepancy, namely {\em global
optimization} and {\em greedy optimization}. The latter is closely connected to {\em herding} algorithms~\citep{Welling2009}.

\paragraph{Global Adaptive Sequences.\\}
The task is posed in terms minimization of the box discrepancy
function~\eqref{eq:discrepancy_gaussian} over the space of sequences of $s$
vectors in $\reals^d$:
\begin{equation*}
S^* = \argmin_{S=(\vv{w}_1\ldots \vv{w}_s)\in \reals^{ds}} \Dbox(S)~.
\end{equation*}
The gradient can be plugged into any first order numerical solver for
non-convex optimization. We use non-linear conjugate gradient in
our experiments (Section~\ref{sec:empirical_adaptive_qmc}).

The above learning mechanism can be extended in various directions. For example, QMC sequences for $n$-point rank-one Lattice
Rules~\citep{DKS13} are integral fractions of a lattice defined by a single generating vector $\vv{v}$. This generating
vector may be learnt via local minimization of the box discrepancy.

\paragraph{Greedy Adaptive Sequences.\\}
Starting with $S_0 = \emptyset$, for $t\geq 1$,
let $S_t = \{\w_1, \ldots, \w_t\}$.
At step $t+1$, we solve the following optimization problem,
\begin{equation}
\w_{t+1} = \argmin_{\w \in \reals^{d}} \Dbox(S_t \cup \{\w\})~.\label{eq:greedy}
\end{equation}
Set $S_{t+1} = S_t \cup \{w_{t+1}\}$ and repeat the above procedure.
The gradient of the above objective is also given in \eqref{eq:gradient}.
Again, we use non-linear conjugate gradient in
our experiments (Section~\ref{sec:empirical_adaptive_qmc}).

The greedy adaptive procedure is closely related to the herding algorithm, recently
presented by~\citet{Welling2009}. Applying the herding algorithm to $PW_{\b}$ and $p(\cdot)$, and using
our notation, the points $\w_1, \w_2, \dots$ are generated using the following iteration
\begin{eqnarray*}
\w_{t+1} & \in & \arg\max_{\w \in \reals^d} \langle \z_t(\cdot), h(\w,\cdot) \rangle_{PW_{\b}} \\
\z_{t+1}(\x) &\equiv& \z_t(\x) + \mu_{h,p}(x) - h(\w,\x)~.
\end{eqnarray*}
In the above, $\z_0,\z_1,\dots$ is a series of functions in $PW_{\b}$. The literature is not
specific on the initial value of $\z_0$, with both $\z_0 = 0$ and $\z_0 = \mu_{h,p}$ suggested.
Either way, it is always the case that $\z_t = \z_0 + t(\mu_{h,p} - \hat{\mu}_{h,p,S_t})$ where
$S_t=\{\w_1,\dots,\w_t\}$.

\citet{KernelHerding} showed that under some additional assumptions, the herding algorithm,
when applied to a RKHS $\H$, greedily minimizes $\left\Vert \mu_{h,p} - \hat{\mu}_{h,p,S_t} \right\Vert^2_\H$,
which, recall, is equal to $D_{h,p}(S_t)$. Thus, under certain assumptions, herding and~\eqref{eq:greedy} are equivalent.
\citet{KernelHerding} also showed that under certain restrictions on the RKHS, herding will reduce the discrepancy
in a ratio of $O(1/t)$. However, it is unclear whether those restrictions hold for $PW_{\b}$ and $p(\cdot)$.
Indeed,~\citet{BachHerding} recently shown that these restrictions never hold for infinite-dimensional RKHS,
as long as the domain is compact. This result does not immediately apply to our case since $\reals^d$ is not compact.

\paragraph{Weighted Sequences.\\}
Classically, Monte-Carlo and Quasi-Monte Carlo approximations of integrals are unweighted, or more precisely, 
have a uniform weights. However, it is quite plausible to weight the approximations, 
i.e. approximate $I_d[f] = \int_{[0,1]^d} f(\x)d\x$
using 
\begin{equation}
I_{S,\Xi}[f]=\sum^n_{i=1}\xi_if(\w_i)~,\label{eq:weighted_I}
\end{equation}
where $\Xi = \{\vv{\xi}_1,\dots,\vv{\xi}_s\}\subset \reals$ is a set of weights. This lead to the feature map
$$\hat{\Psi}_S(\vv{x}) = \left[ \sqrt{\xi_1} e^{-i \vv{x}^T \vv{w}_1}\ldots \sqrt{\xi_s} e^{-i \vv{x}^T \vv{w}_s} \right]~.$$
This construction requires $\xi_i \geq 0$ for $i=1,\dots,s$, although we note that~\eqref{eq:weighted_I} itself
does not preclude negative weights. We do not require the weights to be normalized, that is it is possible 
that $\sum^s_{i=1}\xi_i \neq 1$.

One can easily generalize the result of the previous section to derive the following discrepancy measure that takes
into consideration the weights
\begin{eqnarray*}
D^{\Box \b}_{p}(S,\Xi)^2 &=& \pi^{-d} \prod_{j=1}^d \int_{-b_j}^{b_j} \lvert
\varphi_j(\beta) \rvert^2 d\beta - \nonumber \\
& & 2(2\pi)^{-d} \sum_{l=1}^s
\xi_l \prod_{j=1}^d \int_{-b_j}^{b_j}\varphi_j(\beta)
e^{i w_{lj}\beta}d\beta + \nonumber\\
& & \sum_{l=1}^s
\sum_{j=1}^s \xi_l\xi_j\sinc_\b(\w_l,\w_j)~.
\end{eqnarray*}
Using this discrepancy measure, global adaptive and greedy adaptive sequences of points and weights can be found.

However, we note that if we fix the points, then optimizing just the weights is a simple convex optimization problem.
The box discrepancy can be written as 
$$
D^{\Box \b}_{p}(S,\Xi)^2 = \pi^{-d} \prod_{j=1}^d \int_{-b_j}^{b_j} \lvert
\varphi_j(\beta) \rvert^2 d\beta - 2 \vv{v}^T \vv{\xi} + \xi^T \vv{H} \vv{\xi}\,,
$$
where $\vv{\xi}\in\reals^s$ has entry $i$ equal to $\xi_i$, $\vv{v}\in\reals^s$ 
and $\vv{H}\in\reals^{s \times s}$ are defined by 
$$\vv{H}_{ij}=\sinc_\b(\w_l,\w_j)$$
$$\vv{v}_i=(2\pi)^{-d}\prod_{j=1}^d \int_{-b_j}^{b_j}\varphi_j(\beta) e^{i w_{lj}\beta}d\beta\,.$$
The $\vv{\xi}$ that minimizes $D^{\Box \b}_{p}(S,\Xi)^2$ is equal to $\vv{H}^{-1}\vv{v}$, but there is no guarantee that 
$\xi_i \geq 0$ for all $i$. We need to explicitly impose these conditions.
Thus, the optimal weights can be found by solving the following convex optimization problem
\begin{equation}
\Xi^* = \argmin_{\vv{\xi}\in \reals^s} \xi^T \vv{H} \vv{\xi} - 2 \vv{v}^T \vv{\xi} ~~~~\text{s.t.}~~\vv{\xi} \geq 0~.
\label{eq:weighted}
\end{equation}

Selecting the weights in such a way is closely connected to the so-called {\em Bayesian Monte Carlo} (BMC) method, 
originally suggested by~\citet{BMC}. In BMC, a Bayesian approach is utilized in which the function 
is a assumed to be random with a prior that is a Gaussian Process. Combining with the observations,
a posterior is obtained, which naturally leads to the selection of weights.
\citet{HerdingBayesian} subsequently pointed out the connection between this approach and the
herding algorithm discussed earlier. 

We remark that as long as all the weights are positive, the hypothesis space of functions induced by 
the feature map (that is, $\{ g_\w (\x) = \hat{\Psi}^T_S(\vv{x}) \w,~\w\in\reals^s\}$) will
not change in terms of the set of functions in it. However, the norms will be affected 
(that is, the norm of a function in that set also depends on the weights), 
which in turn affects the regularization. 

%% file: empirical.tex

In this section we report experiments with both classical QMC sequences and
adaptive sequences learnt from box discrepancy minimization.

\subsection{Experiments With Classical QMC Sequences}

We examine the behavior of classical low-discrepancy sequences when compared to
random Fourier features (i.e., MC).  We consider four sequences: Halton, Sobol',
Lattice Rules, and Digital Nets. For Halton and Sobol', we use the
implementation available in
MATLAB.\footnote{\href{}{http://www.mathworks.com/help/stats/quasi-random-numbers.html}}
For Lattice Rules and Digital Nets, we use publicly available
implementations.\footnote{\href{}{http://people.cs.kuleuven.be/~dirk.nuyens/qmc-generators/}}
For all four low-discrepancy sequences,  we use scrambling and shifting techniques
recommended in the  QMC literature (see \citet{DKS13} for details). For Sobol',
Lattice Rules and Digital Nets, scrambling introduces randomization and hence
variance. For Halton sequence, scrambling is deterministic, and there is no
variance. The generation of these sequences is extremely fast, and quite
negligible when compared to the time for any reasonable downstream use.
For example, for \texttt{census} dataset with size 18,000 by 119,
if we choose the number of random features $s=2000$,
the running time for performing kernel ridge regression model is more than 2 minutes,
while the time of generating the QMC sequences is only around 0.2 seconds (Digital Nets 
sequence takes longer, but not much longer) and that of MC sequence is around 0.01 seconds.
Therefore, we do not report running times as these are essentially the same
across methods.

In all experiments, we work with a Gaussian kernel. For learning, we use regularized least square 
classification on the feature mapped dataset, which can  be thought of as a form of 
approximate kernel ridge regression. For each dataset, we performed 5-fold cross-validation 
when using random Fourier features (MC sequence) to set the bandwidth $\sigma$,
and then used the same $\sigma$ for all other sequences. 

\paragraph{Quality of Kernel Approximation\\}
In our setting, the most natural and
fundamental metric for comparison is the quality of approximation of the Gram matrix. We examine how close $\tilde{\vv{K}}$ (defined by
$\tilde{\vv{K}}_{ij}=\tilde{k}(\x_i,\x_j)$ where $\tilde{k}(\cdot, \cdot)=\langle \hat{\Psi}_S(\cdot),
\hat{\Psi}_S(\cdot) \rangle$ is the kernel approximation) is to the Gram matrix
$\vv{K}$ of the exact kernel.

We examine four datasets: \texttt{cpu} (6554 examples, 21 dimensions), \texttt{census}
(a subset chosen randomly with 5,000 examples, 119 dimensions),
\texttt{USPST} (1,506 examples, 250 dimensions after PCA)
and \texttt{MNIST} (a subset chosen randomly with 5,000 examples, 250 dimensions after PCA).
The reason we do subsampling on large datasets is to be able to compute the full
exact  Gram matrix for comparison purposes. The reason we use dimensionality
reduction on \texttt{MNIST} is that the maximum dimension supported by the Lattice Rules implementation we use is 250.

To measure the quality of approximation we
use both $\|\vv{K} - \tilde{\vv{K}}\|_2 / \|\vv{K}\|_2$ and $\|\vv{K} - \tilde{\vv{K}}\|_F / \|\vv{K}\|_F$.  The plots are shown in Figure~\ref{fig:qmc_gram}.
\begin{figure*}[t]
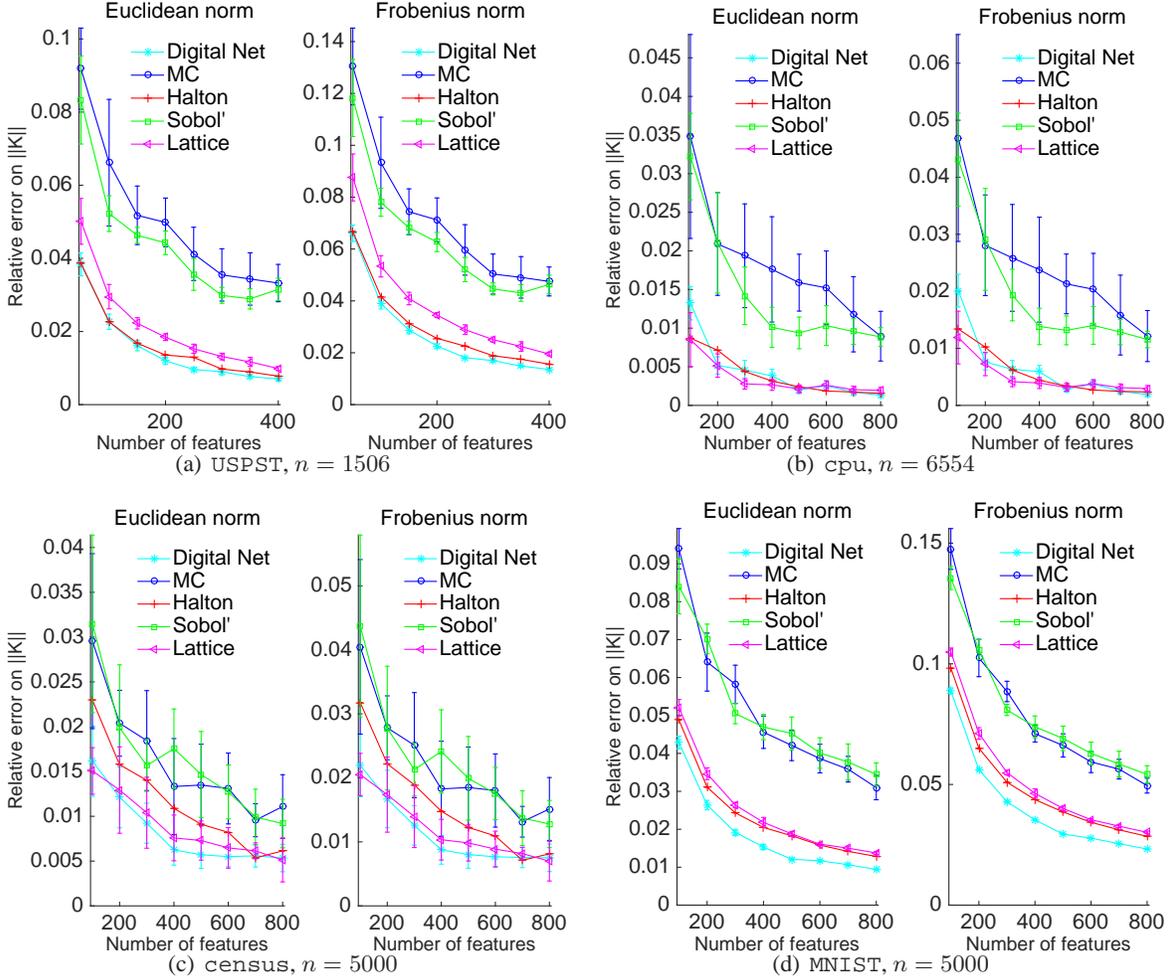

\centering
\begin{tabular}{cc}
 \subfigure[\texttt{USPST}, $n = 1506$]{
 \includegraphics[width=0.45\textwidth]{figures/qmc_gram_fro_USPST-pca_for_revision}} &
 \subfigure[\texttt{cpu}, $n = 6554$]{
 \includegraphics[width=0.45\textwidth]{figures/qmc_gram_fro_cpu_for_revision}} \\
 \subfigure[\texttt{census}, $n = 5000$]{
 \includegraphics[width=0.45\textwidth]{figures/qmc_gram_fro_census_for_revision}} &
 \subfigure[\texttt{MNIST}, $n = 5000$]{
 \includegraphics[width=0.45\textwidth]{figures/qmc_gram_fro_mnist-pca_for_revision}}
\end{tabular}
\caption{Relative error on approximating the Gram matrix measured in Euclidean norm and Frobenius norm,
     i.e., $\|\vv{K} - \tilde{\vv{K}}\|_2 / \|\vv{K}\|_2$ and $\|\vv{K} - \tilde{\vv{K}}\|_F / \|\vv{K}\|_F$, for various $s$.
     For each kind of random feature and $s$, 10 independent trials are executed,
     and the mean and standard deviation are plotted.
}
\label{fig:qmc_gram}
\end{figure*}

{\em We can clearly see that except Sobol' sequences classical low-discrepancy sequences consistently
produce better approximations to the Gram matrix than the approximations produced using MC sequences.} Among the four classical
QMC sequences, the Digital Nets, Lattice Rules and Halton sequences yield much lower error. Similar results were observed for other datasets (not reported here).
Although using scrambled variants of QMC sequences may incur some variance, the variance is quite small compared to that of the MC random features.

Scrambled (whether deterministic or randomized) QMC sequences tend to yield higher accuracies than non-scambled QMC sequences. In Figure~\ref{fig:random_qmc}, we show the ratio between the relative errors achieved by using both scrambled and non-scrambled QMC sequences. As can be seen, scrambled QMC sequences provide more accurate approximations in most cases as the ratio value tends to be less than one. In particular, scrambled Lattice sequence outperforms the non-scrambled one across all the cases for larger values of $s$. Therefore, in the rest of the experiments we use scrambled sequences.

\begin{figure}
\noindent \begin{centering}
\begin{tabular}{cccc}
\subfigure[\texttt{USPST}]{
 \includegraphics[width=0.23\textwidth]{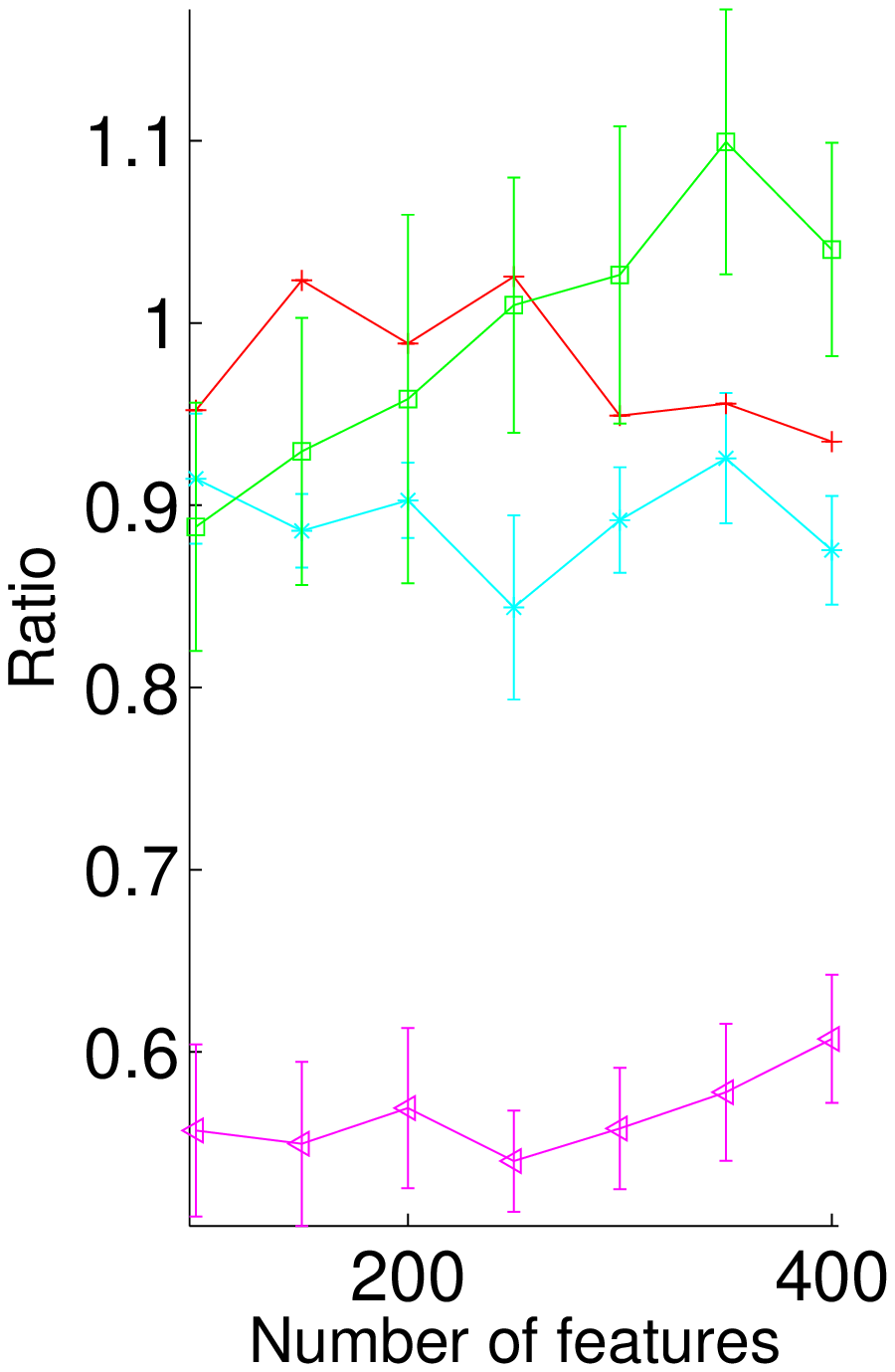}} &
\subfigure[\texttt{CPU}]{
 \includegraphics[width=0.23\textwidth]{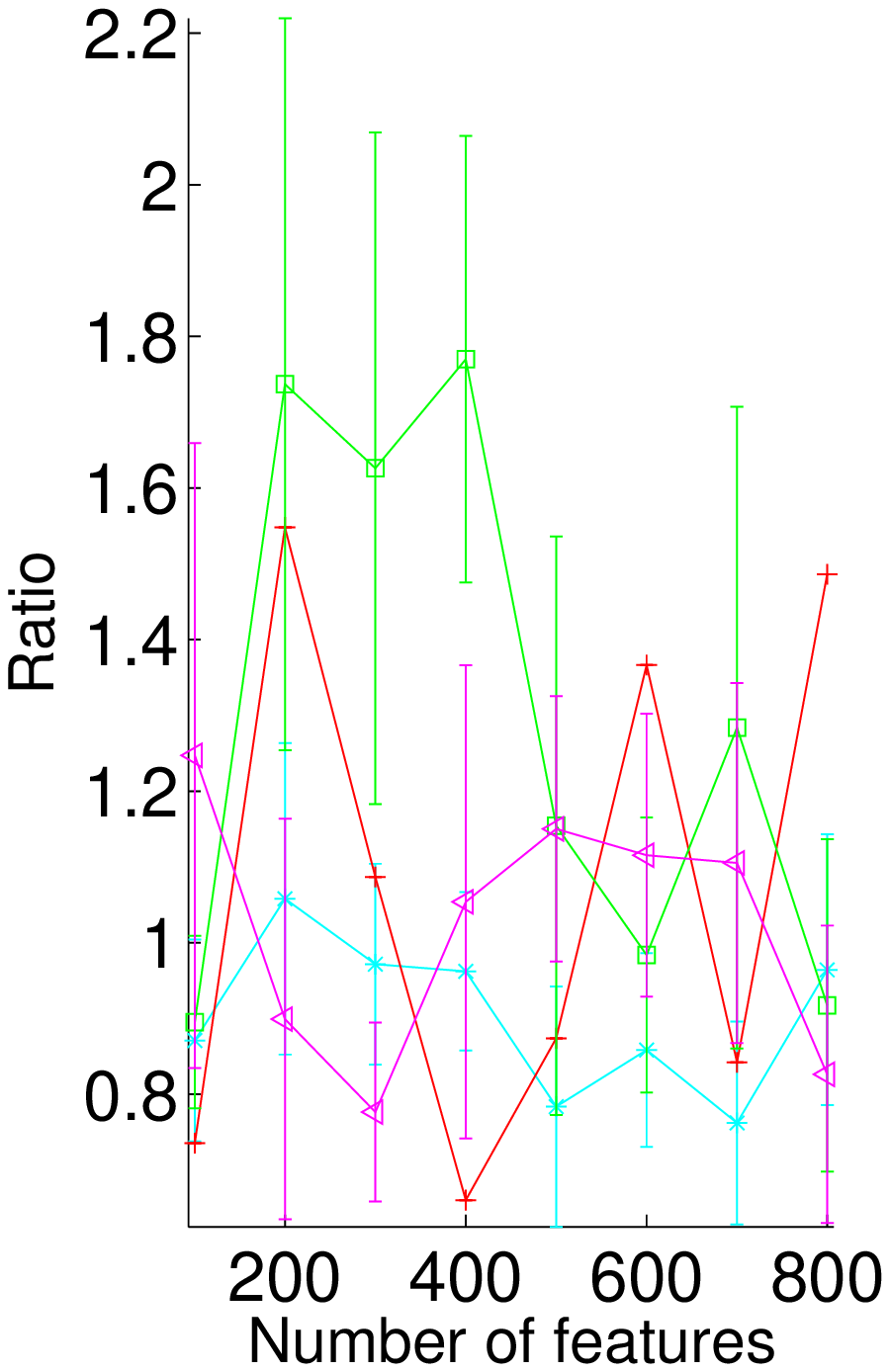}}
\subfigure[\texttt{CENSUS}]{
 \includegraphics[width=0.23\textwidth]{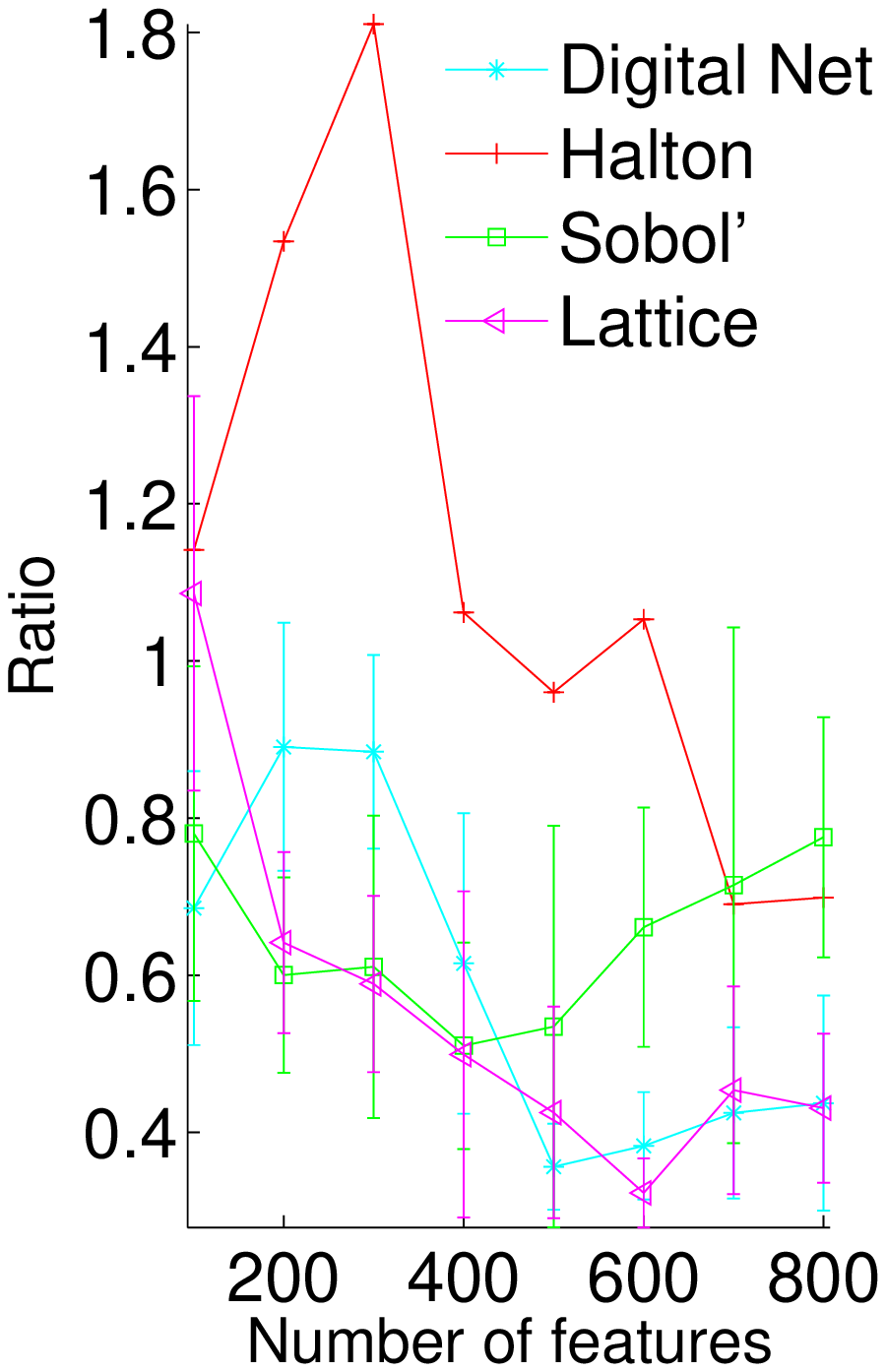}} &
\subfigure[\texttt{MNIST}]{
 \includegraphics[width=0.23\textwidth]{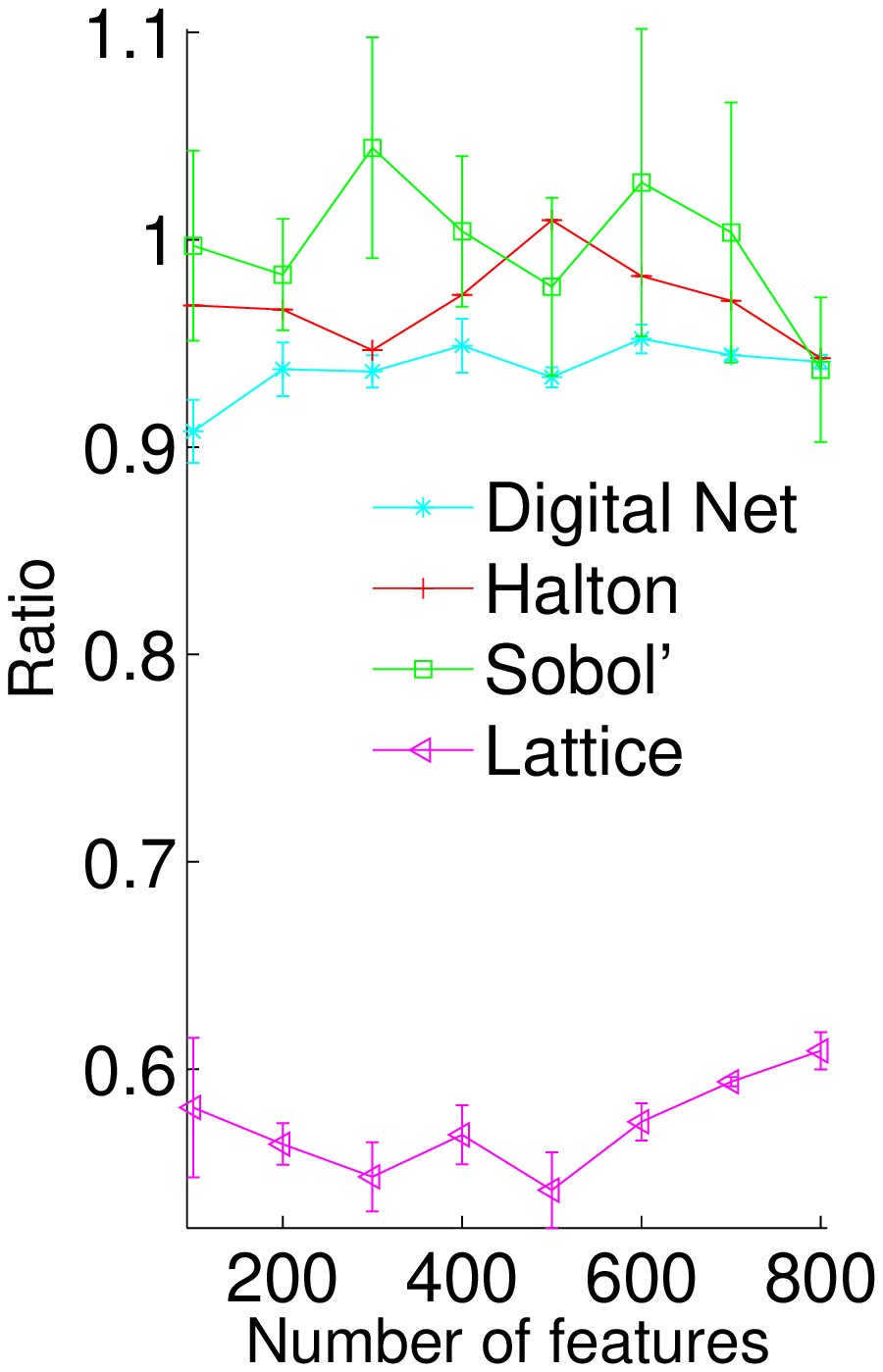}}
\end{tabular}
\par\end{centering}
\caption{Ratio between relative errors on approximating the Gram matrix using both the scrambled and non-scambled version of the same QMC sequence for various $s$. The lower the ratio value is, the more accurate the scrambled QMC approximation is. For each kind of QMC sequences and $s$, $10$ independent trails are executed, and the mean and standard deviation are plotted.}
\label{fig:random_qmc}
\end{figure}

\paragraph{Does better Gram matrix approximation translate to lower generalization
errors?\\}
We consider two regression datasets, \texttt{cpu} and \texttt{census}, and  use (approximate) kernel
ridge regression to build a regression model.
The ridge parameter is set by the optimal value we obtain via 5-fold cross-validation on the training set by using the MC sequence. Table~\ref{table:downstream} summarizes the results.

\begin{table}[h!tpb]
\begin{center}
\begin{sc}
\begin{tabular}{c|c|cccc|c}
    & $s$ & Halton & Sobol' & Lattice & Digit & MC \\
    \hline
 \multirow{6}{*}{\rotatebox{90}{cpu}} &  \multirow{2}{*} {100}  & {\bf 0.0367} &
 0.0383 & 0.0374  & 0.0376 & 0.0383 \\
  & & (0) & (0.0015) & (0.0010)  & (0.0010)  & (0.0013)  \\
 & \multirow{2}{*} {500} & {\bf 0.0339} & 0.0344 & 0.0348  & 0.0343 & 0.0349 \\
 & &  (0) & (0.0005) & (0.0007)  & (0.0005)  & (0.0009)  \\
 & \multirow{2}{*} {1000} & {\bf 0.0334} & 0.0339 & 0.0337  & 0.0335 & 0.0338 \\
 & &  (0) & (0.0007) & (0.0004)  & (0.0003)  & (0.0005)  \\
 \hline
 \multirow{6}{*}{\rotatebox{90}{census}} &  \multirow{2}{*} {400} & {\bf 0.0529}
 & 0.0747 & 0.0801  & 0.0755 & 0.0791 \\
 & & (0) & (0.0138) & (0.0206)  & (0.0080)  & (0.0180)  \\
  &  \multirow{2}{*} {1200} & {\bf 0.0553} & 0.0588 & 0.0694  & 0.0587 & 0.0670
  \\
 & & (0) & (0.0080) & (0.0188)  & (0.0067)  & (0.0078)  \\
 &  \multirow{2}{*} {1800} & {\bf 0.0498} & 0.0613 & 0.0608  & 0.0583 & 0.0600
 \\
 & & (0) & (0.0084) & (0.0129)  & (0.0100)  & (0.0113)  \\
\end{tabular}
\end{sc}
\end{center}
\caption{Regression error, i.e., $\|\hat{\y} - \y\|_2 / \|\y\|_2$ where $\hat{\y}$ is the predicted value
and $\y$ is the ground truth.  For each kind of random feature and $s$, 10 independent trials are executed,
and the mean and standard deviation are listed.}
 \label{table:downstream}
\end{table}

As we see, for \texttt{cpu}, all the sequences behave similarly, with
the Halton sequence yielding the lowest test error.
For \texttt{census}, the advantage of using Halton sequence is significant
(almost 20\% reduction in generalization error) followed by Digital Nets and Sobol'.
In addition, MC sequence tends to generate higher variance across all the sampling size.
Overall, QMC sequences, especially Halton, outperform MC sequences on these datasets.


When performed on classification datasets by using the same learning model, with a moderate range of $s$, e.g., less than
$2000$, the QMC sequences do not yield accuracy improvements over the MC sequence with the same consistency as in the
regression case. The connection between kernel approximation and the performance in downstream applications is outside the
scope of the current paper. Worth mentioning in this regard, is the recent work by \citet{Bach_sharp}, which analyses the
connection between Nystr\"{o}m approximations of the Gram matrix, and the regression error, and the work of~\citet{EM14}
on kernel methods with statistical guarantees. 

\paragraph{Behavior of Box Discrepancy\\}
Next, we examine if $\Dbox$ is predictive
of the quality of approximation. We compute the normalized square box discrepancy values
(i.e., $\pi^{d}(\prod^d_{j=1} b_j)^{-1} \Dbox(S)^2$) as well as Gram matrix approximation error for the
different sequences with different sample sizes $s$. The expected normalized square box discrepancy 
values for MC are computed using~\eqref{eq:expec_mc_disc}.

Our experiments revealed that using the full $\Box \b$ does not yield box discrepancy values
that are very useful. Either the values were not predictive of the kernel approximation, or
they tended to stay constant. Recall, that while the
bounding box $\Box\b$ is set based on observed ranges of feature values in the
dataset, the actual distribution of points $\bar{\X}$ encountered inside that
box might be far from uniform. This lead us to consider the discrepancy measure 
when measured on the central part of the bounding box (i.e., $\Box\b/2$ instead
of $\Box\b$), which is equal to the integration error averaged over that part of
the bounding box. Presumably, points from $\bar{\X}$ concentrate in that
region, and they may be more relevant for downstream predictive task.

The results are shown in Figure~\ref{fig:qmc-bound}. 
In the top graphs we can see, as expected, increasing number of features in the sequence leads to a 
lower box discrepancy value.
In the bottom graphs, which compare $\|\vv{K} - \tilde{\vv{K}}\|_F / \|\vv{K}\|_F$
to $D^{\Box \b/2}$, we can see a strong correlation between the quality of
approximation and the discrepancy value.

\begin{figure}
\noindent \begin{centering}
\begin{tabular}{cc}
\subfigure[\texttt{cpu}, $D^{\Box \b/2}$ versus $s$]{
\includegraphics[width=0.45\textwidth]{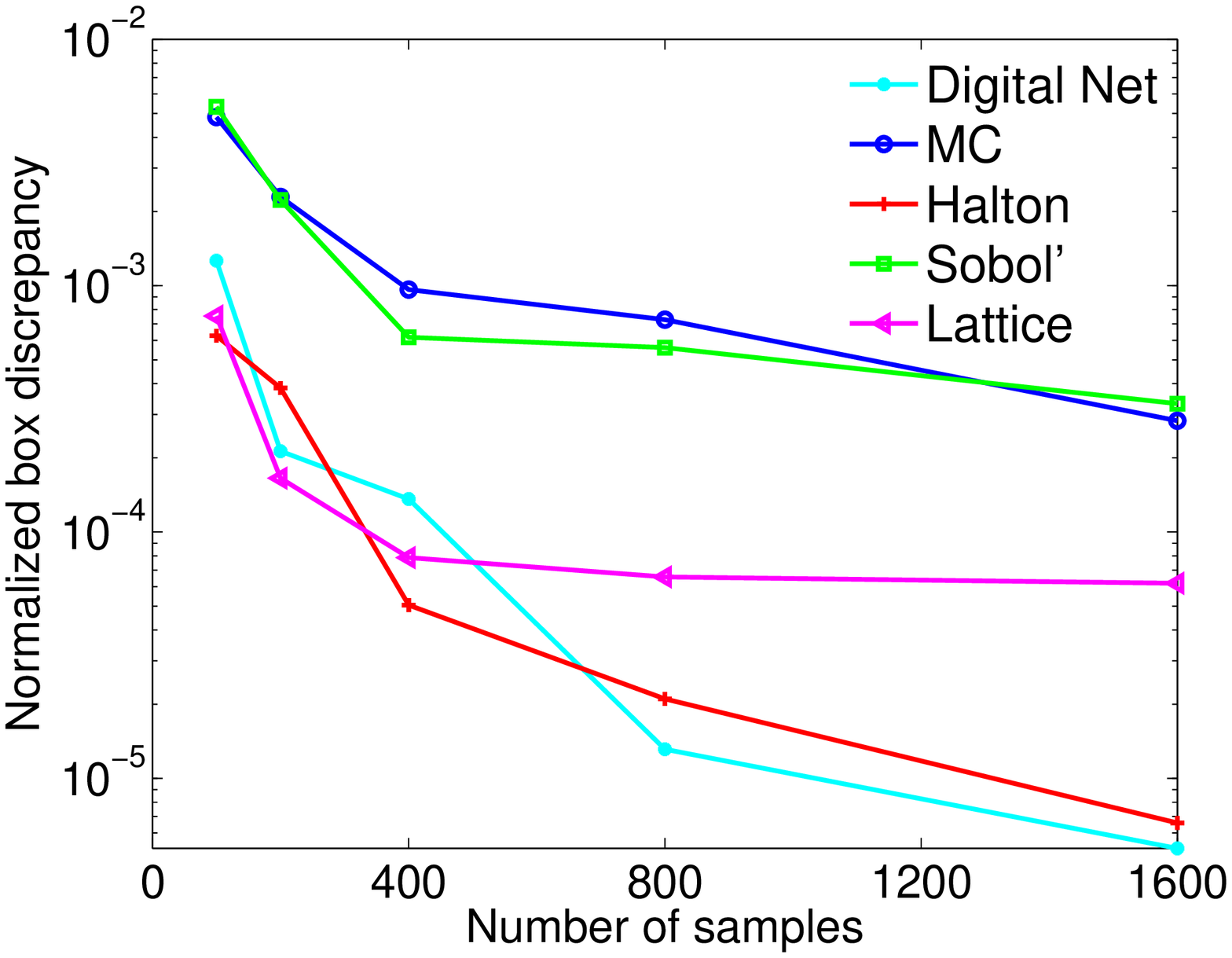}} &
\subfigure[\texttt{census}, $D^{\Box \b/2}$ versus $s$]{
\includegraphics[width=0.45\textwidth]{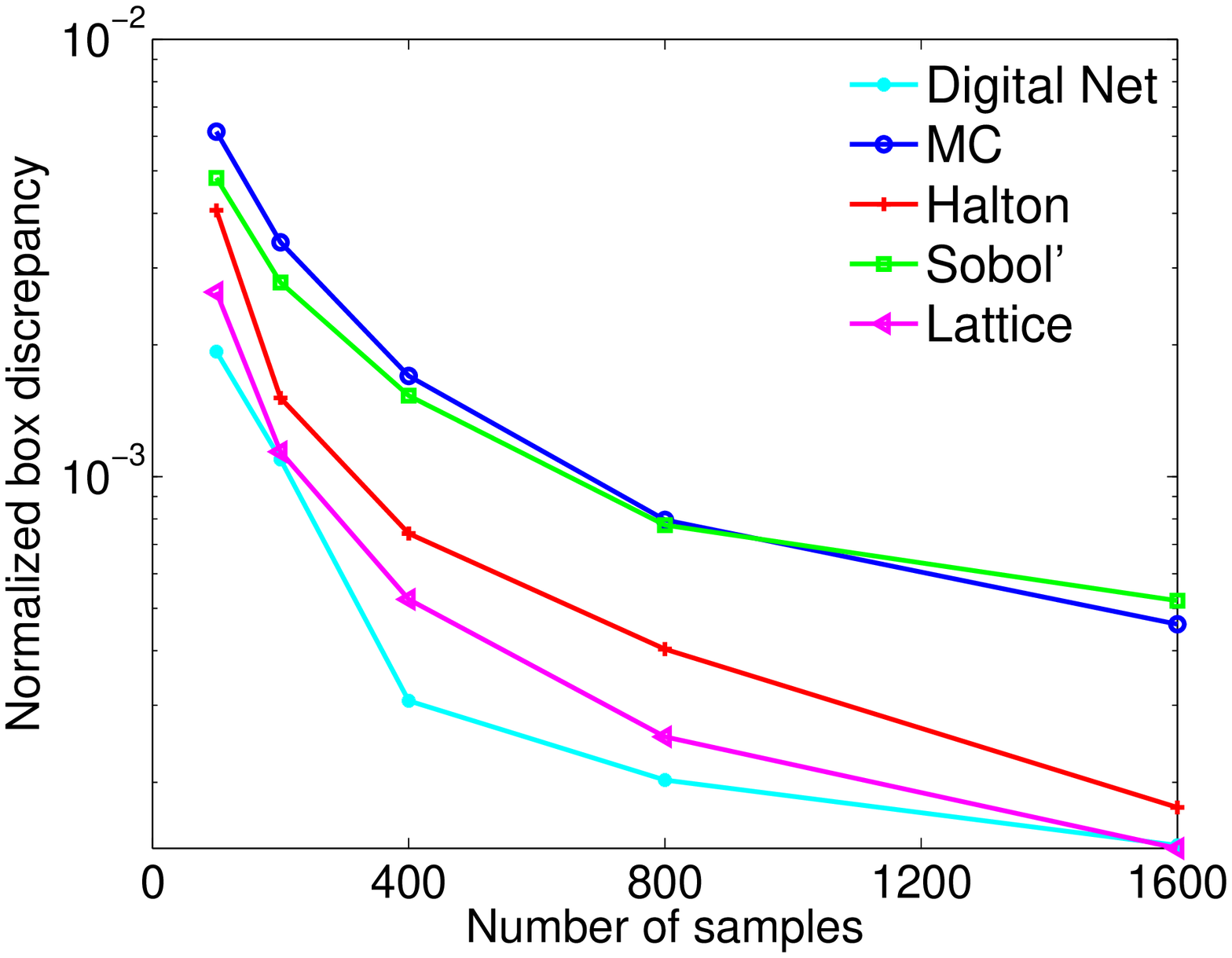} } \\
\subfigure[\texttt{cpu}, relative error versus $D^{\Box \b/2}$]{
\includegraphics[width=0.45\textwidth]{figures/qmc_disc_gram_fro_cpu_for_revision} } &
\subfigure[\texttt{census}, relative error versus $D^{\Box \b/2}$]{
\includegraphics[width=0.45\textwidth]{figures/qmc_disc_gram_fro_census_for_revision} }
\end{tabular}
\par\end{centering}
\caption{Discrepancy values ($D^{\Box \b/2}$) for the different sequences on \texttt{cpu} and \texttt{census}. We measure the discrepancy on the central part of the bounding box (we use $\Box\b/2$ instead of
$\Box\b$ as the domain in the box discrepancy).}
\label{fig:qmc-bound}
\end{figure}

\subsection{Experiments With Adaptive QMC}\label{sec:empirical_adaptive_qmc}

The goal of this subsection is to provide a proof-of-concept for learning
 adaptive QMC sequences, using the three schemes
described in Section~\ref{sec:qmc_adaptive}. We demonstrate that QMC sequences can be improved to produce better
approximation to the Gram matrix, and that can sometimes lead to improved
generalization error.

 
Note that the running time of learning the adaptive sequences is less relevant in our experimental setting for the following
reasons. Given the values of $s$, $d$, $\b$ and $\sigma$ the optimization of a sequence needs only to be done once. There is
some flexibility in these parameters: $d$ can be adjusted by adding zero features or by doing PCA on the input; one can use
longer or shorter sequences; and the data can be forced to a fit a particular bounding box using (possibly non-equal) scaling
of the features (this, in turn, affects the choice of the $\sigma$) . Since designing adaptive QMC sequences is
data-independent with applicability to a variety of downstream applications of kernel methods, it is quite conceivable to
generate many point sets in advance and to use them for many learning tasks. Furthermore, the total size of the sequences
($s \times d$) is independent of the number of examples $n$, which is the dominant term in large scale learning settings.

We name the three sequences as {\it Global Adaptive}, {\it Greedy Adaptive} and {\it Weighted} respectively.
For {\it Global Adaptive},
the Halton sequence is used as the initial setting of the optimization variables $S$.
For {\it Greedy Adaptive}, when optimizing for $\w_t$,
the $t$-th point in the Halton sequence is used as the initial point.
In both cases, we use non-linear conjugate gradient
to perform numerical optimization.
For {\it Weighted}, the initial features are generated using the Halton sequence and we optimize for the weights.
We used \texttt{CVX}~\citep{cvx,gb08} to compute the sequence (solve~\eqref{eq:weighted}).

\paragraph{Integral Approximation\\} We begin by examining the integration error over the unit square by using three
different sequences, namely, MC, Halton and global adaptive QMC sequences. The integral is of the form $\int_{[0,1]^2} e^{-i \u^T
\t} d\t$ where $\u$ spans the unit square. The error is plotted in Figure~\ref{fig:int_err}. We see that MC sequences
concentrate most of the error reduction near the origin. The Halton sequence gives significant improvement expanding the
region of low integration error. Global adaptive QMC sequences give another order of magnitude improvement in integration error
which is now diffused over the entire unit square; the estimation of such sequences is ``aware" of the full integration
region. In fact, by controlling the box size (see plot labeled b/2), adaptive sequences can be made to focus in a specified
sub-box which can help with generalization if the actual data distribution is better represented by this sub-box.
\begin{figure*}[t] \centering
 \includegraphics[width=0.75\textwidth]{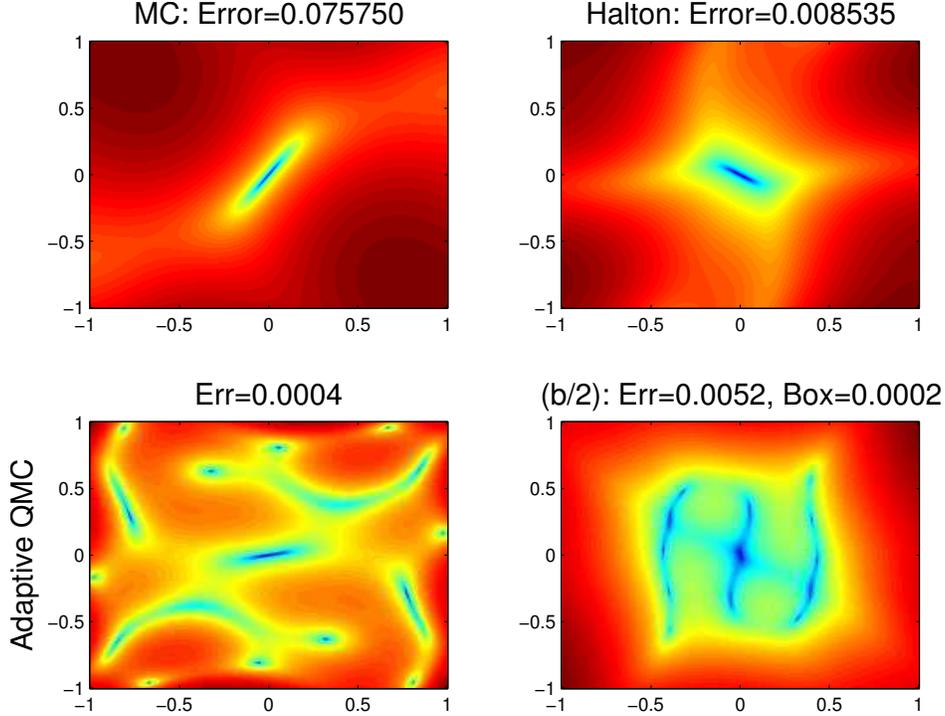}
\caption{Integration error.
}
\label{fig:int_err}
\end{figure*}

\paragraph{Quality of Kernel Approximation\\}

\begin{figure*}[t]
\noindent \begin{centering}
\begin{tabular}{ccc}
\includegraphics[width=0.3\textwidth]{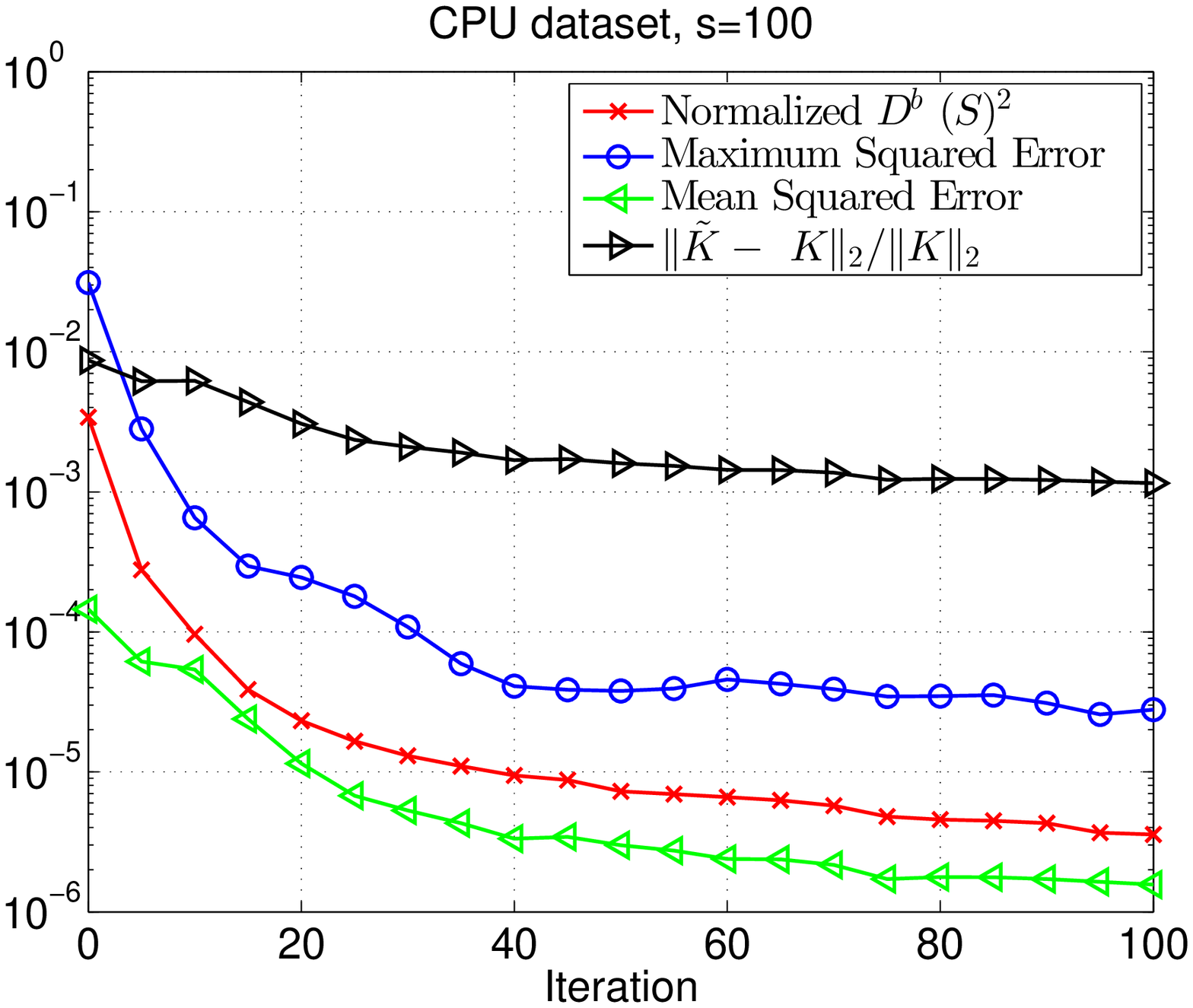} &
\includegraphics[width=0.3\textwidth]{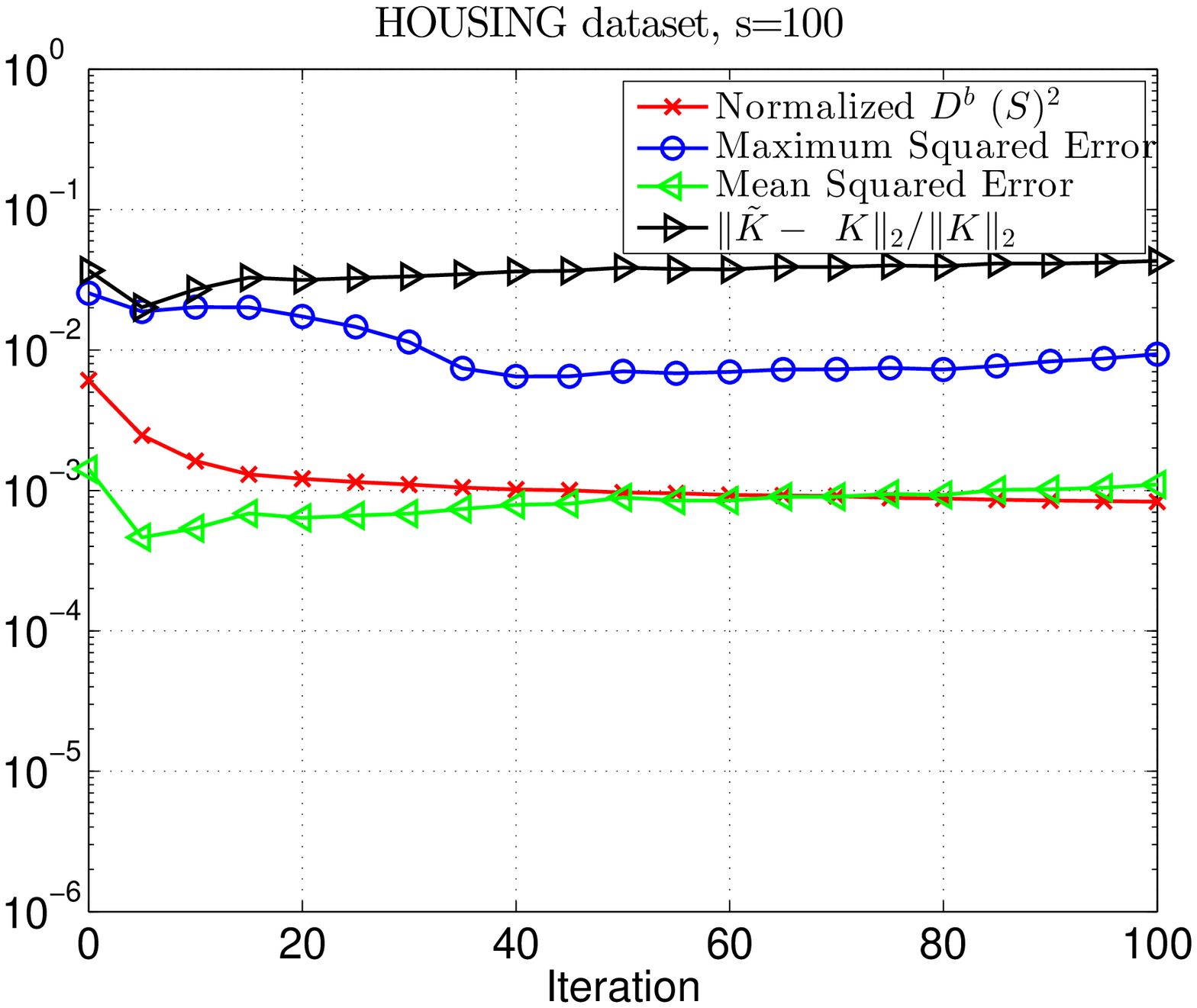}&
\includegraphics[width=0.3\textwidth]{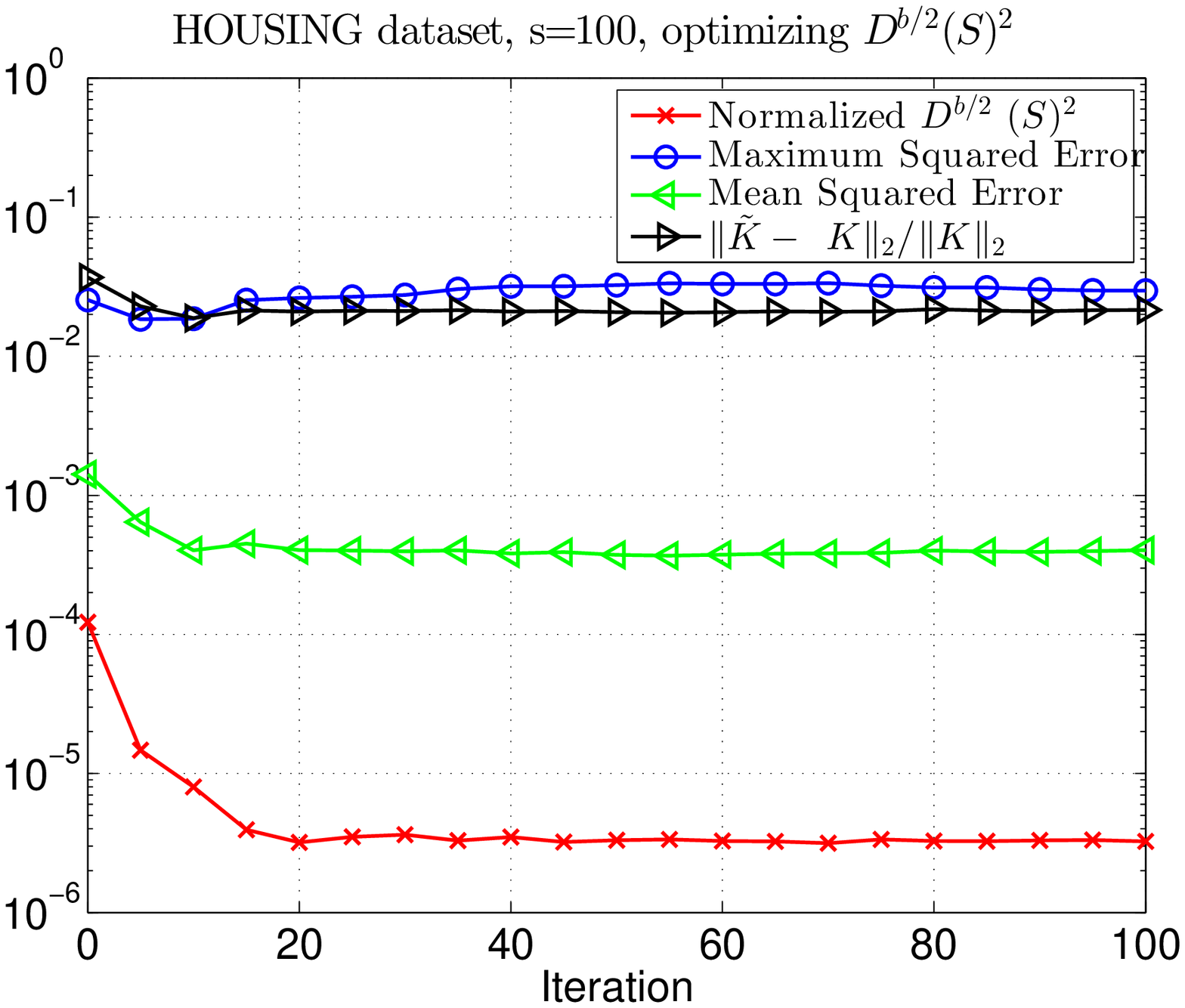}
\tabularnewline \textbf{(a)} & \textbf{(b)} & \textbf{(c)} \tabularnewline
\end{tabular}
\par\end{centering}
\caption{\label{fig:global_adaptive} Examining the behavior of learning {\it Global Adaptive} sequences.
Various metrics on the Gram matrix approximation are plotted.
}
\end{figure*}

\begin{figure*}[t]
\noindent \begin{centering}
\begin{tabular}{ccc}
\includegraphics[width=0.3\textwidth]{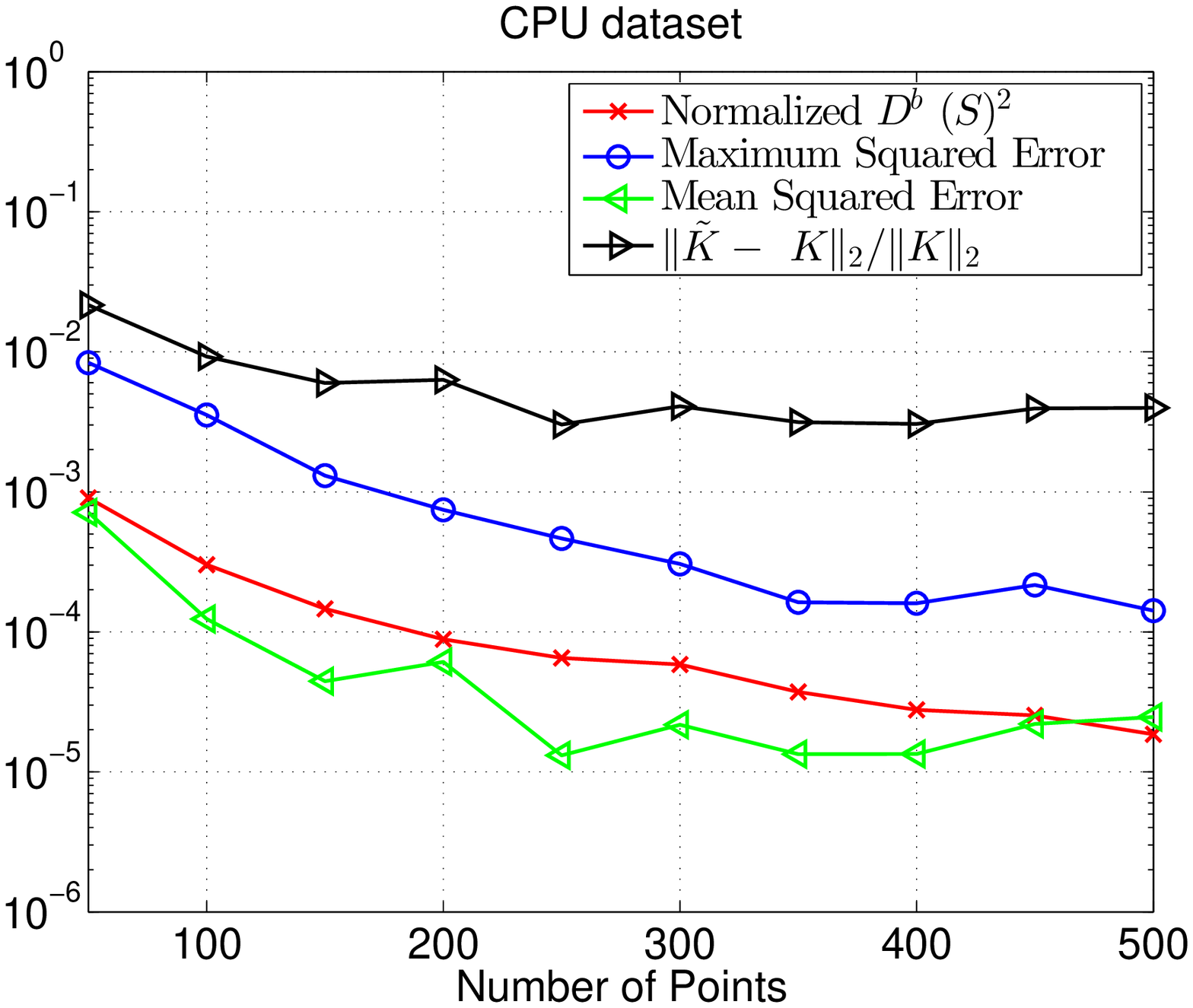} &
\includegraphics[width=0.3\textwidth]{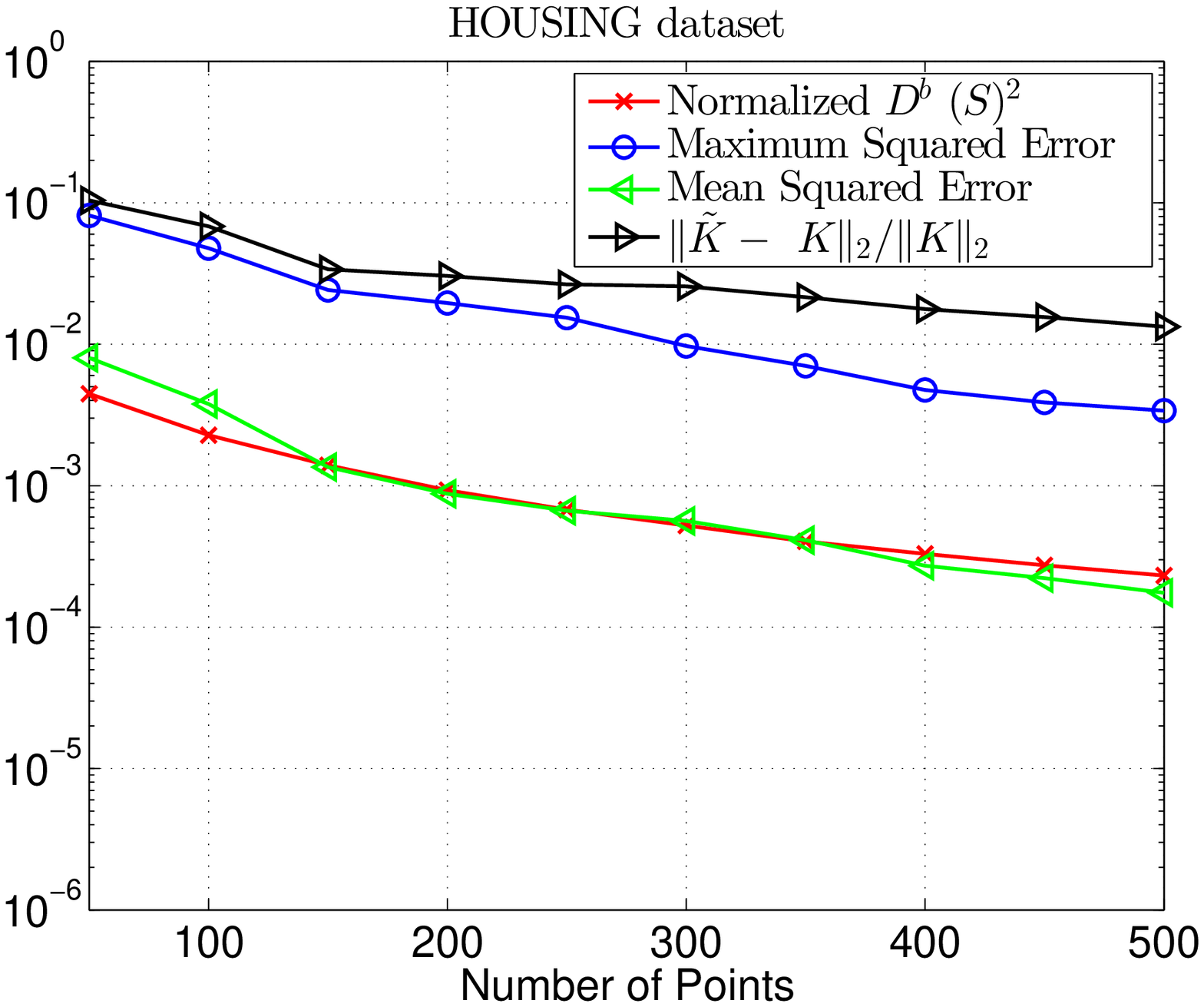}&
\includegraphics[width=0.3\textwidth]{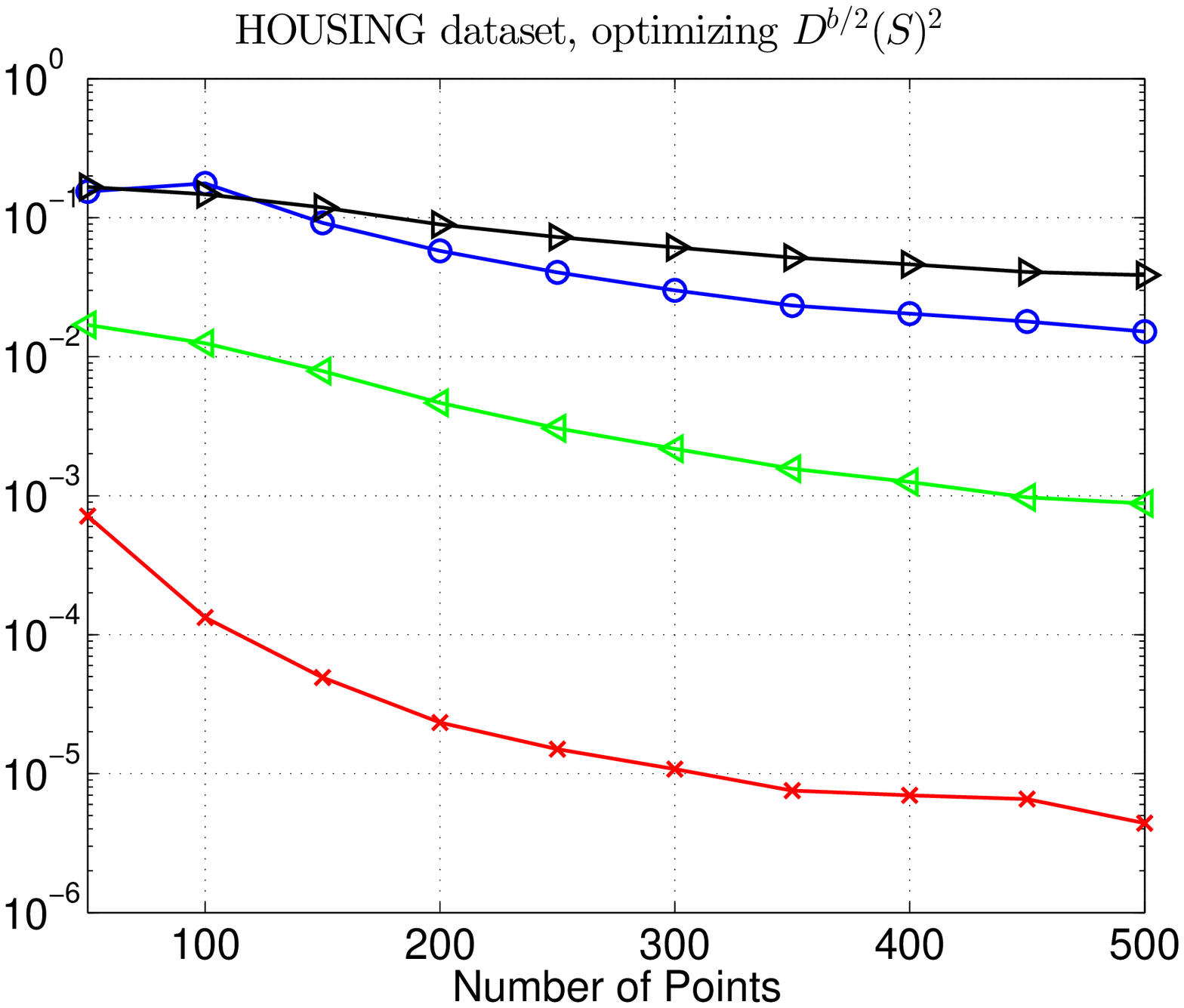}
\tabularnewline \textbf{(a)} & \textbf{(b)} & \textbf{(c)} \tabularnewline
\end{tabular}
\par\end{centering}
\caption{\label{fig:greedy_adaptive} Examining the behavior of learning {\it Greedy Adaptive} sequences.
Various metrics on the Gram matrix approximation are plotted.
}
\end{figure*}

In Figure~\ref{fig:global_adaptive} and Figure~\ref{fig:greedy_adaptive} we examine how various metrics (discrepancy,
maximum squared error, mean squared error, norm of the error) on the Gram matrix
approximation evolve during the optimization process for both adaptive sequences.
Since learning the adaptive sequences on dataset with low dimensional features is more affordable, the experiment is performed on two such datasets, namely, \texttt{cpu} and \texttt{housing}.

For {\em Global Adaptive}, we fixed $s=100$ and examine
how the performance evolves as the number of iterations grows.
 In Figure~\ref{fig:global_adaptive} (a)
we examine the behavior on \texttt{cpu}. We see that all metrics go down as the
iteration progresses. This supports our hypothesis that by optimizing the box
discrepancy we can improve the approximation of the Gram matrix.  Figure~\ref{fig:global_adaptive} (b), which
examines the same metrics on the scaled version of the \texttt{housing} dataset,
has some interesting behaviors.
Initially all metrics go down, but eventually all the metrics except the box-discrepancy
start to go up; the box-discrepancy continues to go down. One plausible
explanation is that the integrands are not uniformly distributed in the bounding
box, and that by optimizing the expectation over the entire box we start to
overfit it, thereby increasing the error in those regions of the box where
integrands actually concentrate.
One possible way to handle this is to optimize closer to
the center of the box (e.g., on $\Box \b/2$), under the assumption that integrands
concentrate there. In Figure~\ref{fig:global_adaptive} (c) we try this on
the \texttt{housing} dataset. We see that now the mean error and the norm error are much
improved, which supports the interpretation above.
 But the maximum error eventually goes up. This is quite reasonable as
the outer parts of the bounding box are harder to approximate, so the maximum
error is likely to originate from there. Subsequently, we stop the adaptive learning
of the QMC sequences early, to avoid the actual error from going up due to
averaging.

For {\em Greedy Adaptive}, we examine its behavior as the number of points increases.
In Figure~\ref{fig:greedy_adaptive} (a) and (b),
as expected, as the number of points in the sequence increases,
the box discrepancy goes down.
This is also translated to non-monotonic decrease in the other metrics of Gram matrix approximation.
However, unlike the global case, we see in Figure~\ref{fig:greedy_adaptive} (c),
when the points are generated by optimizing on a smaller box $\Box \b/2$,
the resulting metrics become higher for a fixed number of points.
Although the {\em Greedy Adaptive} sequence can be computed faster than the adaptive sequence,
potentially it might need a large number of points to achieve certain low magnitude of discrepancy.
Hence, as shown in the plots, when the number of points is below 500,
the quality of the optimization is not good enough to provide a good approximation the Gram matrix.
For example, one can check when the number of points is 100,
the discrepancy value of the {\em Greedy Adaptive} sequence is higher than that of the {\em Global Adaptive} sequence
with more than 10 iterations.

\begin{table}[h!tpb]
\scriptsize
\begin{center}
\begin{sc}
\begin{tabular}{c|c|cccc|cccc}
  & & \multicolumn{4}{c|}{$D^{\Box \b}$} & \multicolumn{4}{c}{$D^{\Box \b/4}$} \\ \hline
  & $s$ & Halton & Global$_{\b}$ & Greedy$_{\b}$ & Weighted$_{\b}$ & Halton & Global$_{\b/4}$ & Greedy$_{\b/4}$ & Weighted$_{\b/4}$ \\ \hline
 \multirow{3}{*}{\rotatebox{90}{{cpu}}} & 100 & 3.41e-3 & 1.29e-6 & 3.02e-4 & 7.84e-5 & 9.44e-5 & 5.57e-8 & 2.62e-5 & 1.67e-8 \\
 & 300 & 8.09e-4 & 5.14e-6 & 5.85e-5 & 1.45e-6 & 2.57e-5 & 1.06e-7 & 3.08e-6 & 2.93e-9 \\
 & 500 & 2.39e-4 & 2.83e-6 & 1.86e-5 & 3.39e-7 & 7.91e-6 & 2.62e-8 & 1.04e-6 & 2.43e-9 \\
 \hline
 \multirow{5}{*}{\rotatebox{90}{{census}}} & 400 & 2.61e-3 & 9.32e-4 & 7.47e-4 & 8.83e-4 & 5.73e-4 & 2.79e-5 & 2.20e-5 &2.45e-5 \\
 & 800 & 1.21e-3 & 5.02e-4 & 3.33e-4 & 4.91e-4 & 2.21e-4 & 1.12e-5 & 8.04e-6 & 5.46e-6 \\
 & 1200 & 8.27e-4 & 3.41e-4 & 2.06e-4 & 3.39e-4 & 1.39e-4 & 8.15e-6 & 4.23e-6 & 2.29e-6 \\
 & 1800 & 5.31e-4 & 2.17e-4 & 1.27e-4 & 2.31e-4 & 3.79e-5 & 5.59e-6 & 2.63e-6 & 8.37e-7 \\
 & 2200 & 4.33e-4 & 1.73e-4 & 1.01e-4 & 1.87e-4 & 2.34e-5 & 3.35e-6 & 1.95e-6 & 4.93e-7 
\end{tabular}
\end{sc}
\end{center}
\caption{
Discrepancy values, measured on the full bounding box and its central part, i.e., $D^{\Box \b}$ and $D^{\Box \b/4}$. }
 \label{table:downstream3}
\end{table}

Table~\ref{table:downstream3} also shows the discrepancy values of various sequences on \texttt{cpu} and \texttt{census}. 
Using adaptive sequences improves the discrepancy  values by orders-of-magnitude. 
We note that a significant reduction in terms of discrepancy values can be achieved using only weights, sometimes yielding discrepancy
values that are better than the hard-to-compute global or greedy sequences.

\paragraph{Generalization Error\\}
We use the three algorithms for learning adaptive sequences as described in the previous subsections,
and use them for doing approximate kernel ridge regression.
The ridge parameter is set by the value which is near-optimal for both sequences in 5-fold cross-validation on the training set.
Table~\ref{table:downstream2} summarizes the results.

\begin{table}[h!tpb]
\scriptsize
\begin{center}
\begin{sc}
\begin{tabular}{c|c|ccccccc}
    & $s$ & Halton & Global$_{\b}$ & Global$_{\b/4}$ & Greedy$_{\b}$ & Greedy$_{\b/4}$ & Weighted$_{\b}$ & Weighted$_{\b/4}$ \\
    \hline
  \multirow{3}{*}{\rotatebox{90}{{cpu}}} &  100  & 0.0304 & 0.0315 & {\bf 0.0296} & 0.0307 & {\bf 0.0296} & 0.0366 & 0.0305  \\
    & 300 & 0.0303 & 0.0278 & 0.0293 & 0.0274 & {\bf 0.0269} & 0.0290 & 0.0302  \\
    & 500 & 0.0348 & 0.0347 & 0.0348 & 0.0328 & {\bf 0.0291} & 0.0342 & 0.0347 \\
   \hline
 \multirow{5}{*}{\rotatebox{90}{{census}}} &  400  & 0.0529 & 0.1034 & 0.0997 & 0.0598 & 0.0655 & {\bf 0.0512} & 0.0926 \\
    &  800  & 0.0545 & 0.0702 & 0.0581 & 0.0522 & 0.0501 & {\bf 0.0476} & 0.0487 \\
    &  1200  & 0.0553 & 0.0639 & {\bf 0.0481} & 0.0525 & 0.0498 & 0.0496 & 0.0501 \\
    & 1800 & 0.0498 & 0.0568 & {\bf 0.0476} & 0.0685 & 0.0548 & 0.0498 & 0.0491 \\
    & 2200 & 0.0519 & {\bf 0.0487} & 0.0515 & 0.0694 & 0.0504 & 0.0529 & 0.0499 
\end{tabular}
\end{sc}
\end{center}
\caption{Regression error, i.e., $\|\hat{\y} - \y\|_2 / \|\y\|_2$ where $\hat{\y}$ is the predicted value
and $\y$ is the ground truth. }
 \label{table:downstream2}
\end{table}

For both \texttt{cpu} and \texttt{census},
at least one of the adaptive sequences sequences can yield lower test error for each sampling size (since the test error is already low, around 3\% or 5\%, such improvement in accuracy is not trivial).
For \texttt{cpu}, greedy approach seems to give slightly better results.
When $s=500$ or even larger (not reported here),
the performance of the sequences are very close.
For \texttt{census},
the weighted sequence yields the lowest generalization error when $s=400,800$.
Afterwards we can see global adaptive sequence outperforms the rest of the sequences, even though it has 
better discrepancy values. In some cases, adaptive sequences sometimes produce errors that
are bigger than the unoptimized sequences.

In most cases, the adaptive sequence on the central part of the bounding box outperforms the adaptive sequence
on the entire box. This is likely due to the non-uniformity phenomena discussed
earlier.


%% file: conclusion.tex
Recent work on applying kernel methods to very large datasets, has shown their ability 
to achieve state-of-the-art accuracies that sometimes match those attained
by Deep Neural Networks (DNN) \citep{KernelTIMIT}. Key to these results is the ability 
to apply kernel method to such datasets. The random features approach, originally
due to~\citet{RahimiRecht}, as emerged as a key technology for scaling up 
kernel methods~\citep{KernelsHilbertJSM}.

Close examination of those empirical results reveals that to achieve state-of-the-art
accuracies, a very large number of random features was needed. For example, on 
\texttt{TIMIT}, a classical speech recognition dataset, over 200,000 random features 
were used in order to match DNN performance~\citep{KernelTIMIT}.
It is clear that improving the efficiency of random features can have a significant
impact on our ability to scale up kernel methods, and potentially get even 
higher accuracies. 

This paper is the first to exploit high-dimensional approximate integration
techniques from the QMC literature in this context, with
promising empirical results backed by rigorous theoretical analyses. 
Avenues for future work include incorporating
stronger data-dependence in the estimation of adaptive sequences and analyzing
how resulting Gram matrix approximations translate into downstream performance
improvements for a variety of large-scale learning tasks.

%% file: supplementary2.tex
In this section we give detailed proofs of the assertions made in Section~\ref{sec:qmc_theory} and~\ref{sec:qmc_adaptive}.

\subsection{Proof of Proposition~\ref{prop:unbounded-vhk}}
Recall,
for any $\vv{t} \in \reals^d$, for $\Phi^{-1}(\vv{t})$, we mean $\left(\Phi_1^{-1}(t_1), \ldots, \Phi_d^{-1}(t_d) \right) \in \reals^d $,
where $\Phi_j(\cdot)$ is the CDF of $p_j(\cdot)$.

From $f_{\vv{u}}(\vv{t}) = e^{-i \vv{u}^T \Phi^{-1}(\vv{t})}$,
for any $j = 1,\ldots,d$, we have
$$  \frac{\partial f(\vv{t})}{\partial \vv{t}_j} = (-i) \frac{u_j}{p_j(\Phi_j^{-1}(t_j))} e^{-i \vv{u}^T \Phi_j^{-1}(\vv{t})}~. $$
Thus,
$$  \frac{\partial^d f(\vv{t})}{\partial t_1 \cdots \partial t_d} = \prod_{j=1}^{d} \left( (-i) \frac{u_j}{p_j(\Phi_j^{-1}(t_j))} \right) e^{-i \vv{u}^T \Phi_j^{-1}(\vv{t})}~. $$
In \eqref{eq:Vhk}, when $I = [d]$,
\begin{eqnarray}
 \int_{[0,1]^{|I|}}\left| \frac{\partial f}{\partial \vv{u}_I} \right| d\vv{t}_I
 &=& \int_{[0,1]^{d}}\left|  \frac{\partial^d f(\vv{t})}{\partial t_1 \cdots \partial t_d} \right| d t_1\cdots d t_d \nonumber \\
 &=&  \int_{[0,1]^{d}}\left| \prod_{j=1}^{d} \left( (-i) \frac{u_j}{p_j(\Phi_j^{-1}(t_j))} \right) e^{-i \vv{u}^T \Phi_j^{-1}(\vv{t})} \right| d t_1\cdots d t_d \nonumber \\
 &=&  \int_{[0,1]^{d}} \prod_{j=1}^{d}  \left|  \frac{u_j}{p_j(\Phi_j^{-1}(t_j))} \right| d t_1\cdots d t_d \nonumber \\
 &=&  \prod_{j=1}^d \left( \int_{[0,1]}\left| \frac{u_j}{p_j(\Phi_j^{-1}(t_j))}  \right| d t_j \right)~. \label{eq:unbdd_tmp}
\end{eqnarray}
With a change of variable, $\Phi_j(t_j) = v_j$, for $j=1,\ldots,d$, \eqref{eq:unbdd_tmp} becomes
 \begin{equation*}
  \prod_{j=1}^d \left( \int_{[0,1]}\left| \left( \frac{u_j}{p_j(\Phi_j^{-1}(t_j))} \right) \right| d t_j \right)
 =    \prod_{j=1}^d \left( \int_{\reals} | u_j | d v_j \right) = \infty~.
 \end{equation*}
As this is a term in \eqref{eq:Vhk}, we know that $V_{HK}[ f_{\vv{u}}(\vv{t})]$ is unbounded.

\subsection{Proof of Proposition~\ref{prop:int-err-rkhs}}

We need the following lemmas, across which we share some notation.
\def\EfX{\Expect{f(X)}}
\begin{lemma}
\label{lem:rkhs_finite_integral}
Assuming that $\kappa = \sup_{\x\in \reals^d} h(\x, \x) < \infty$, if $f\in \H$, where $\H$ is an RKHS with
kernel $h(\cdot, \cdot)$, the integral $\int_{\reals^d} f(\x) p(\x) d\x$ is
finite.
\end{lemma}
\begin{proof}
For notational convenience, we  note that $$\int_{\reals^d} f(\x) p(\x) d\x =
\Expect{f(X)},$$  where $\Expect{\cdot}$ denotes expectation and  $X$ is a
random variable distributed according to the probability density $p(\cdot)$ on
$\reals^d$.

Now consider a linear functional $T$ that maps $f$ to $\Expect{f(X)}$,  i.e.,
\begin{equation} T[f] = \Expect{f(X)}. \label{eq:T}\end{equation}  The linear
functional $T$ is a bounded linear functional on the RKHS $\H$. To see this:
\begin{eqnarray}
|  \EfX | &\leq& \Expect{|f (X) |}~~~~\textrm{(Jensen's Inequality)}\nonumber\\
&=& \Expect {| \langle f, h(X, \cdot)\rangle_\H|}~~~~\textrm{(Reproducing
Property)} \nonumber\\
&\leq& \|f\|_{\H} \Expect{\| h(X, \cdot) \|_{\H}}
~~~~\textrm{(Cauchy-Schwartz)}\nonumber\\
&\leq & \|f\|_{\H} \Expect{\sqrt{h(X,X)}} = \|f\|_{\H} \sqrt{\kappa} <
\infty~.\nonumber
\end{eqnarray} This shows that the integral $\int_{\reals^d} f(\x) p(\x) d\x$
exists.
\end{proof}

\begin{lemma}
\label{lem:rkhs_expe}
The mean $\mu_{h,p}(\u) = \int_{\reals^d} h(\u, \x) p(\x) d\x$ is in $\H$. In
addition, for any $f\in \H$,
\begin{equation}
\Expect{f(X)} =  \int_{\reals^d} f(\x) p(\x) d\x = \langle f, \mu_{h,p}\rangle_{\H}~.\label{eq:swap}
\end{equation}
\end{lemma}
\begin{proof}
From the Riesz Representation Theorem, every bounded linear functional on $\H$
admits an inner product representation. Therefore, for $T$ defined in~\eqref{eq:T}, 
there exists $\mu_{h,p} \in \H$ such that,
$$
T[f] = \Expect{f(X)}  = \langle f, \mu_{h,p} \rangle_{\H}~.
$$
Therefore we have,  $\langle f, \mu_{h,p} \rangle_{\H} = \int f(\x) p(\x) d\x$ for
all $f \in \H$. For any $\z$, choosing $f(\cdot) = h(\z, \cdot)$, where $h(\cdot, \cdot)$ is
the kernel associated with $\H$,  and invoking the reproducing property we see that,
$$
\mu_{h,p}(\z) = \langle h(\z, \cdot), \mu_{h,p} \rangle_{\H} = \int_{\reals^d} h(\z, \x) p(\x) d\x~.
$$
\end{proof}

The proof of Proposition~\ref{prop:int-err-rkhs} follows from the existence
Lemmas above, and the following steps.

 \begin{eqnarray*}
 \epsilon_{S,p}[f] & = &  \left| \int_{\reals^d} f(\x) p(\x) d\x -
                             \frac{1}{s} \sum_{l=1}^s f(\w_l) \right| \\
 					& = & \left| \langle f, \mu_{h,p} \rangle_{\H} - \frac{1}{s} \sum_{l=1}^s \langle
 					f, h(\w_l, \cdot) \rangle_{\H} \right|\\
 					& = & \left| \langle f, \mu_{h,p} - \frac{1}{s} \sum_{l=1}^s h(\w_l, \cdot)
 					\rangle_{\H} \right| \\
 					& \leq & \left\|f\right\|_{\H} \left\|\mu_{h,p} - \frac{1}{s} \sum_{l=1}^s h(\w_l,
 					\cdot)\right\|_{\H} = \left\|f\right\|_{\H} D_{h,p} (S)~,
 \end{eqnarray*}
where $D_{h,p}(S)$ is given as follows, \begin{eqnarray*} D_{h,p}(S)^2
&=& \left\|\mu_{h,p} - \frac{1}{s} \sum_{l=1}^s h(\w_l, \cdot)\right\|^2_{\H}\nonumber\\
&=& \langle \mu_{h,p}, \mu_{h,p} \rangle_{\H} - \frac{2}{s} \sum_{l=1}^s \langle \mu_{h,p}, h(\w_l,
\cdot) \rangle_{\H}  + \frac{1}{s^2} \sum_{l=1}^s \sum_{j=1}^s \langle h(\w_l,
\cdot), h(\w_j, \cdot)
\rangle_{\H}\nonumber\\
&=& \Expect{\mu_{h,p}(X)} -   \frac{2}{s} \sum_{l=1}^s \Expect{h(\w_l,
\cdot)} + \frac{1}{s^2} \sum_{l=1}^s \sum_{j=1}^s h(\w_l, \w_j) \nonumber\\
&=& \int _{\reals^d}\int_{\reals^d} h(\omega, \phi) p(\omega)
p(\phi) d\omega d\phi -
\frac{2}{s}\sum^s_{l=1} \int_{\reals^d} h(\w_l, \omega) p(\omega) d\omega  \nonumber \\
 & & \qquad + \frac{1}{s^2}\sum_{l=1}^s \sum_{j=1}^s
h(\w_l,\w_j)~.
\end{eqnarray*}

\subsection{Proof of Theorem~\ref{thm:discr-pw}}
We apply \eqref{eq:worstcase_error} to the particular case of $h=\sinc_\b$. We have
\begin{eqnarray*}
\int_{\reals^d}\int_{\reals^d} h(\omega, \phi) p(\omega) p(\phi) d\omega d\phi & = &
    \pi^{-d} \int_{\reals^d}\int_{\reals^d} \prod^d_{j=1} \frac{\sin(b_j (\omega_j - \phi_j))}{\omega_j - \phi_j}
    p_j(\omega_j)p_j(\phi_j) d\omega d\phi \\
    & = &
    \pi^{-d}  \prod^d_{j=1}\int _\reals\int_\reals \frac{\sin(b_j (\omega_j - \phi_j))}{\omega_j - \phi_j}
    p_j(\omega_j)p_j(\phi_j) d\omega_j d\phi_j~,
\end{eqnarray*}
and
\begin{eqnarray*}
\sum^s_{l=1} \int_{\reals^d} h(w_l, \omega) p(\omega) d\omega & = &
    \pi^{-d} \sum^s_{l=1} \int_{\reals^d} \prod^d_{j=1} \frac{\sin(b_j (w_{lj} - \omega_j))}{w_{lj} - \omega_j} p_j(\omega_j) d\omega \\
    & = &
    \pi^{-d}  \sum^s_{l=1} \prod^d_{j=1} \int_{\reals^d} \frac{\sin(b_j (w_{lj} - \omega_j))}{w_{lj} - \omega_j} p_j(\omega_j) d\omega_j~.
\end{eqnarray*}
So we can consider each coordinate on its own.

Fix $j$. We have
\begin{eqnarray*}
\int_\reals \frac{\sin(b_j x)}{x}p_j(x) dx & = & \int_\reals \int_0^{b_j} \cos(\beta x)p_j(x) d\beta dx \\
                                    & = & \frac{1}{2} \int_{-b_j}^{b_j} \int_\reals  e^{i \beta x}p(x) dx d\beta \\
                                    & = & \frac{1}{2} \int_{-b_j}^{b_j}\varphi_j(\beta) d\beta~.
\end{eqnarray*}
The interchange in the second line is allowed since the $p_j(x)$ makes the function integrable (with respect to $x$).

Now fix $w\in\reals$ as well. Let $h_j(x,y)=\sin(b_j(x-y))/\pi(x-y)$. We have
\begin{eqnarray*}
\int_\reals h_j(\omega, w) p_j(\omega) d\omega & = & \pi^{-1} \int_\reals \frac{\sin(b_j(\omega-w))}{\omega-w} p_j(\omega) d\omega \\
                                          & = & \pi^{-1} \int_\reals \frac{\sin(b_j x)}{x} p_j(x+w) dx \\
                                          & = & (2\pi)^{-1} \int_{-b_j}^{b_j}\varphi_j(\beta) e^{iw\beta}d\beta~,
\end{eqnarray*}
where the last equality follows from first noticing that the characteristic function associated with the density function
$x \mapsto p_j(x+w)$ is $\beta \mapsto\varphi(\beta) e^{iw\beta}$, and then applying the previous inequality.

We also have,
\begin{eqnarray*}
\int_\reals \int_\reals \frac{\sin(b_j(x-y))}{x-y}p_j(x)p_j(y) dx dy & = &
       \int_\reals \int_\reals \int_0^{b_j}\cos(\beta(x-y))p_j(x)p_j(y) d\beta dx dy \\
& = &  \frac{1}{2} \int_\reals \int_\reals \int_{-b_j}^{b_j} e^{i\beta(x-y)}p_j(x)p_j(y) d\beta dx dy \\
& = &  \frac{1}{2} \int_{-b_j}^{b_j}\int_\reals \int_\reals e^{i\beta(x-y)}p_j(x)p_j(y) dx dy d\beta \\
& = &  \frac{1}{2} \int_{-b_j}^{b_j} \left( \int_\reals e^{i \beta x}p_j(x) dx \right) \left( \int_\reals e^{-i\beta y}p_j(y)  dy \right) d\beta \\
& = &  \frac{1}{2} \int_{-b_j}^{b_j}\varphi_j(\beta) \varphi_j(\beta)^* d\beta \\
& = &  \frac{1}{2} \int_{-b_j}^{b_j}\lvert \varphi_j(\beta) \rvert^2 d\beta~.
\end{eqnarray*}
The interchange at the third line is allowed because of $p_j(x)p_j(y)$. In the last line we use the fact that the $\varphi_j(\cdot)$ is Hermitian.

\subsection{Proof of Theorem~\ref{thm:average-case}}
Let $b >0$ be a scalar, and let $u \in [-b, b]$ and $z \in \reals$. We have,
\begin{eqnarray*}
\int^{\infty}_{-\infty} e^{-iux} \frac{\sin(b(x-z))}{\pi(x-z)} dx & = &
    e^{-iuz} \int^{\infty}_{-\infty} e^{-i2\pi\frac{u}{2b}y} \frac{\sin(\pi y)}{\pi y} dy \\
    & = & e^{-iuz} \rect(u/2b) \\
    & = &  e^{-iuz}~.
\end{eqnarray*}
In the above, $\rect$ is the function that is 1 on $[-1/2, 1/2]$ and zero elsewhere.

The last equality implies that for every $\u \in \Box \b$ and every $\x \in \reals^d$ we have
$$
f_\u(\x) = \int_{\reals^d} f_\u(\y)\sinc_\b(\y,\x) d\y~.
$$

We now have for every $\u \in \Box \b$,
\begin{eqnarray*}
\epsilon_{S,p}[f_\u] & = &  \left| \int_{\reals^d} f_\u(\x) p(\x) d\x -
                            \frac{1}{s} \sum_{i=1}^s f(\w_i) \right| \\
& = & \left| \int_{\reals^d} \int_{\reals^d} f_\u(\y)\sinc_\b(\y,\x) d\y p(\x) d\x -
                            \frac{1}{s} \sum_{i=1}^s \int_{\reals^d} f_\u(\y)\sinc_\b(\y,\w_i) d\y \right| \\
& = & \left| \int_{\reals^d} f_\u(\y) \left[ \int_{\reals^d} \sinc_\b(\y,\x) p(\x) d\x -
                            \frac{1}{s} \sum_{i=1}^s  \sinc_\b(\y,\w_i) \right] d\y \right|~.
\end{eqnarray*}
Let us denote
$$
r_S (\y) = \int_{\reals^d} \sinc_\b(\y,\x) p(\x) d\x - \frac{1}{s} \sum_{i=1}^s  \sinc_\b(\y,\w_i)~.
$$
So,
$$
\epsilon_{S,p}[f_\u] = \left| \int_{\reals^d} f_\u(\y) r_S(\y) d\y \right|~.
$$

The function $r_S(\cdot)$ is square-integrable, so it has a Fourier transform $\hat{r}_S(\cdot)$. The above formula is exactly
the value of $\hat{r}_S(\u)$. That is,
$$
\epsilon_{S,p}[f_\u] = \left| \hat{r}_S(\u) \right|~.
$$
Now,
\begin{eqnarray*}
\ExpectSub{f\sim {\cal U}(\F_{\Box \b})}{\epsilon_{S,
p} [f]^2} & = & \ExpectSub{\u\sim {\cal U}(\Box \b)}{\epsilon_{S,
p} [f_\u]^2}\\
& = & \int_{\u \in \Box \b} \left| \hat{r}_S(\u) \right|^2 \left( \prod^d_{j=1} 2b_j \right)^{-1} d\u \\
& = & \left( \prod^d_{j=1} 2b_j \right)^{-1} \| \hat{r}_S \|^2_{L2} \\
& = & \frac{(2\pi)^{d}}{\prod^d_{j=1} 2b_j} \| r_S \|^2_{PW_\b} \\
& = & \frac{\pi^{d}}{\prod^d_{j=1} b_j} \Dbox_p(S)^2~.
\end{eqnarray*}
The equality before the last follows from Plancherel formula and the equality of the norm in $PW_\b$ to the $L2$-norm. The last equality follows from the fact that $r_S$ is exactly the expression used in the proof of Proposition~\ref{prop:int-err-rkhs}
to derive $\Dbox_p$.

\subsection{Proof of Corollary~\ref{cor:discr-gaussian}}
In this case, $p(\vv{x}) = \prod_{j=1}^d p_j(x_j)$ where $p_j(\cdot)$ is the density function of $\mathcal{N}(0,1/\sigma_j)$.
The characteristic function associated with $p_j(\cdot)$ is $\varphi_j(\beta) = e^{-\frac{\beta^2}{2\sigma_j^2}}$.
We apply \eqref{eq:discr-sinc} directly.

For the first term, since
\begin{eqnarray*}
    \int_0^{b_j} \lvert \varphi_j(\beta) \rvert^2 d\beta
    &=& \int_0^{b_j} e^{-\frac{\beta^2}{\sigma_j^2}} d\beta \\
    &=& \sigma_j \int_0^{b_j/\sigma_j} e^{-y^2} dy  \\
    &=& \frac{\sigma_j \sqrt{\pi}}{2} \erf\left(\frac{ b_j}{\sigma_j} \right)~,
\end{eqnarray*}
we have
\begin{equation}
     \label{eq:discr_gaussian_1}
   \pi^{-d} \prod_{j=1}^d \int_0^{b_j} \lvert\varphi_j(\beta) \rvert^2 d\beta =  \prod_{j=1}^d \frac{\sigma_j}{2 \sqrt{\pi}}  \erf\left(\frac{ b_j}{\sigma_j} \right)~.
 \end{equation}

For the second term, since
 \begin{eqnarray*}
   \int_{-b_j}^{b_j}\varphi_j(\beta) e^{i w_{lj}\beta} d\beta
   &=&  \int_{-b_j}^{b_j}  e^{-\frac{\beta_j^2}{2\sigma^2} + i w_{lj}\beta} d\beta  \\
   &=&  e^{-\frac{\sigma_j w_{lj}}{2}} \int_{-b_j}^{b_j} e^{ -\left( \frac{\beta}{\sqrt{2}\sigma_j} - i \frac{\sigma_j w_{lj}}{\sqrt{2}} \right)^2} d\beta \\
   &=& \sqrt{2}\sigma_j e^{-\frac{\sigma_j^2 w_{lj}^2}{2}} \int_{-\frac{b_j}{\sqrt{2}\sigma_j}}^{\frac{b_j}{\sqrt{2}\sigma_j}} e^{- \left(y - i\frac{\sigma_j w_{lj}}{\sqrt{2}} \right)^2} dy \\
   &=& \sqrt{2}\sigma_j e^{-\frac{\sigma_j^2 w_{lj}^2}{2}} \int_{-\frac{b_j}{\sqrt{2}\sigma_j} - i\frac{\sigma_j w_{lj}}{\sqrt{2} }}^{\frac{b_j}{\sqrt{2}\sigma_j} - i\frac{\sigma_j w_{lj}}{\sqrt{2}}} e^{-z^2} dz  \\
   &=& \frac{\sqrt{\pi}\sigma_j}{\sqrt{2}} e^{-\frac{\sigma_j^2 w_{lj}^2}{2}} \left( \erf\left(-\frac{b_j}{\sqrt{2}\sigma_j} - i\frac{\sigma_j w_{lj}}{\sqrt{2} }\right)
       - \erf\left(\frac{b_j}{\sqrt{2}\sigma_j} - i\frac{\sigma_j w_{lj}}{\sqrt{2} }\right) \right) \\
  &=& \sqrt{2\pi}\sigma_j e^{-\frac{\sigma_j^2 w_{lj}^2}{2}} \Real\left( \erf\left(-\frac{b_j}{\sqrt{2}\sigma_j} - i\frac{\sigma_j w_{lj}}{\sqrt{2} }\right) \right)~,
 \end{eqnarray*}
 we have
 \begin{equation}
   \label{eq:discr_gaussian_2}
   \frac{2}{s} (2\pi)^{-d}  \sum_{l=1}^s \prod_{j=1}^d \int_{-b_j}^{b_j}\varphi_j(\beta) e^{i\w_{lj}\beta}d\beta
   = \frac{2}{s} \sum_{l=1}^s \prod_{j=1}^d \frac{ \sigma_j}{\sqrt{2\pi}} e^{-\frac{\sigma_j^2 w_{lj}^2}{2}} \Real\left( \erf\left(-\frac{b_j}{\sqrt{2}\sigma_j} - i\frac{\sigma_j w_{lj}}{\sqrt{2} }\right) \right)~.
  \end{equation}

Combining \eqref{eq:discr_gaussian_1}, \eqref{eq:discr_gaussian_2} and \eqref{eq:discr-sinc}, \eqref{eq:discrepancy_gaussian} follows.

\subsection{Proof of Corollary~\ref{cor:disc_mc_general}}
The proof is similar to the proof of Theorem 3.6 of \citet{DKS13}.
Notice that since $\sup_{\x \in \reals^d} h(\x, \x) < \infty$, we have $\int_{\reals^d} h(\x,\x) p(\x) d\x < \infty$.
From Lemma~\ref{lem:rkhs_expe} we know that
$ \int_{\reals^d} h(\cdot, \y) p(\y) d\y \in \H$, hence from Lemma~\ref{lem:rkhs_finite_integral},
we have $\int_{\reals^d} \int_{\reals^d} h(\x,\y) p(\x) p(\y) d\x d\y < \infty$.

By \eqref{eq:worstcase_error}, we have
\begin{eqnarray*}
D_{h,p}(S)^2 &=& \int _{\reals^d}\int_{\reals^d} h(\omega, \phi) p(\omega)
p(\phi) d\omega d\phi \\
& &  - \frac{2}{s}\sum^s_{l=1} \int_{\reals^d}
h(\w_l, \omega) p(\omega) d\omega\\
& &  + \frac{1}{s^2}\sum_{l=1}^s
h(\w_l,\w_l) + \frac{1}{s^2}\sum_{l,j=1, l\neq j}^s h(\w_l,\w_j)~.
\end{eqnarray*}
Then,
\begin{equation*}
\begin{split}
\Expect{D_{h,p}(S)^2} =
\int_{[0,1]^d} \cdots \int_{[0,1]^d} \left(
\int _{\reals^d}\int_{\reals^d} h(\omega, \phi) p(\omega)
p(\phi) d\omega d\phi
- \frac{2}{s}\sum^s_{l=1} \int_{\reals^d}
h(\Phi^{-1}(\t_l), \omega) p(\omega) d\omega
\right. \\
\left.
+ \frac{1}{s^2}\sum_{l=1}^s
h(\Phi^{-1}(\t_l), \Phi^{-1}(\t_l)) + \frac{1}{s^2}\sum_{l,j=1, l\neq j}^s h(\Phi^{-1}(\t_l),\Phi^{-1}(\t_j)) \right) d\t_1 \cdots d\t_s
~.
\end{split}
\end{equation*}

Obviously, the first is a constant which is independent to $\t_1,\ldots,\t_s$.
Since all the terms are finite, we can interchange the integral and the sum among rest terms.
In the second term, for each $l$, the only dependence on $\t_1, \ldots, \t_s$ is $\t_l$, hence all the other $\t_j$ can be integrated out.
That is,
\begin{eqnarray*}
 \int_{[0,1]^d} \cdots \int_{[0,1]^d} \int_{\reals^d}
 h(\Phi^{-1}(\t_l), \omega) p(\omega) d\omega d\t_1 \cdots d\t_s &=& \int_{[0,1]^d} \int_{\reals^d}  h(\Phi^{-1}(\t_l), \omega) p(\omega) d\omega d\t_l \\
 &=& \int_{\reals^d} \int_{\reals^d} h(\phi, \omega) p(\phi) p(\omega) d\phi d\omega.
\end{eqnarray*}
Above, the last equality comes from a change of variable, i.e., $\t_l = \left(\Phi_1(\phi_1), \ldots, \Phi_d(\phi_d) \right)$.

Similar operations can be done for the third and fourth term.
Combining all of these, we have the following,
\begin{eqnarray*}
\Expect{D_{h,p}(S)^2} &=&
\int _{\reals^d}\int_{\reals^d} h(\omega, \phi) p(\omega)
p(\phi) d\omega d\phi
- 2 \int _{\reals^d}\int_{\reals^d} h(\omega, \phi) p(\omega)
p(\phi) d\omega d\phi  \\
&& + \frac{1}{s} \int _{\reals^d} h(\omega, \omega) p(\omega) d\omega
+ \frac{s-1}{s} \int _{\reals^d}\int_{\reals^d} h(\omega, \phi) p(\omega)
p(\phi) d\omega d\phi \\
&=& \frac{1}{s} \int _{\reals^d} h(\omega, \omega) p(\omega) d\omega - \frac{1}{s} \int _{\reals^d}\int_{\reals^d} h(\omega, \phi) p(\omega)
p(\phi) d\omega d\phi~.
\end{eqnarray*}

\subsection{Proof of Proposition~\ref{prop:gradient}}
Before we compute the derivative, we prove two auxiliary lemmas.
\begin{lemma}
\label{lem:gradient_1}
Let $\vv{x} \in \reals^d$ be a variable and $\vv{z} \in \reals^d$ be fixed vector.
Then,
  \begin{equation}
     \frac{\partial \sinc_\b(\x, \z)}{\partial x_j} = b_j \sinc'_{b_j}(x_j,z_j) \prod_{q\neq j} \sinc_{b_q}(x_q,z_q)~.
  \end{equation}
\end{lemma}
We omit the proof as it is a simple computation that follows from the definition of $\sinc_\b$.

\begin{lemma}
\label{lem:gradient_2}
The derivative of the scalar function $f(x) = \Real\left[ e^{-a x^2}
\erf\left( c + i d x\right) \right]$, for real scalars $a,c,d$ is given
by, $$\frac{\partial f}{\partial x} = -2ax e^{-ax^2} \Real\left[\erf\left( c + i d x\right) \right]
              + \frac{2d}{\sqrt{\pi}} e^{-ax^2} e^{d^2x^2 - c^2} \sin(2cdx)~. $$
\end{lemma}
\begin{proof}
Since
  \begin{eqnarray}
    f(x) &=& \frac{1}{2} \left( e^{-ax^2} \erf(c+idx) + \left( e^{-ax^2} \erf(c+idx) \right)^* \right) \nonumber \\
         &=&   \frac{1}{2} \left( e^{-ax^2} \erf(c+idx) +  e^{-ax^2} \erf(c-idx) \right)~,
  \end{eqnarray}
it suffices to compute the the derivative $g(x) = e^{-ax^2} \erf(c+idx)$.

Let $k(x) = \erf(c+idx)$. We have
    \begin{equation}
        g'(x) = -2ax e^{-ax^2} k(x) + e^{-ax^2} k'(x)~.
    \end{equation}
Since
  \begin{eqnarray}
       k(x) &=&  \erf(c+idx) \nonumber \\
             &=& \frac{2}{\sqrt{\pi}} \int^{c+idx}_0 e^{-z^2} dz \nonumber \\
     &=& \frac{2}{\sqrt{\pi}} \left( \int^{c}_0 e^{-z^2} dz + \int^{c+idx}_c e^{-z^2} dz \right) \nonumber \\
     &=&  \frac{2}{\sqrt{\pi}} \left( \int^{c}_0 e^{-y^2} dy + (id) \int^x_0 e^{-(c+idt)^2} dt \right)~,
  \end{eqnarray}
  we have
   \begin{equation}
     k'(x) = \frac{2}{\sqrt{\pi}} e^{-(c+idx)^2} =  \frac{2d}{\sqrt{\pi}} e^{d^2x^2 - c^2} ( \sin(2cdx) + i \cos(2cdx) )~.
  \end{equation}

We now have
     \begin{eqnarray}
      f'(x) &=& \frac{1}{2} \left( g'(x) + (g^*(x))' \right)  \nonumber \\
         &=& \frac{1}{2} \left(  g'(x) + (g'(x))^* \right)  \nonumber \\
         &=& \frac{1}{2} \left( -2ax e^{-ax^2} (k(x) + k^*(x)) + e^{-ax^2} ( k'(x) + (k'(x))^* ) \right) \nonumber \\
         &=& \frac{1}{2} \left( -4ax e^{-ax^2} \Real\left[\erf\left( c + i d x\right) \right]
              + e^{-ax^2} \frac{4d}{\sqrt{\pi}} e^{d^2x^2 - c^2} \sin(2cdx)   \right)  \nonumber \\
         &=& -2ax e^{-ax^2} \Real\left[\erf\left( c + i d x\right) \right]
              + \frac{2d}{\sqrt{\pi}} e^{-ax^2} e^{d^2x^2 - c^2} \sin(2cdx)~.
     \end{eqnarray}
\end{proof}

\begin{proof}[Proof of Proposition~\ref{prop:gradient}]
For the first term in \eqref{eq:discrepancy_gaussian},
that is $\frac{1}{s^2}\sum_{m=1}^s \sum_{r=1}^s \sinc_\b(\w_m,\w_r)$,
to compute the partial derivative of $w_{lj}$, we only have to consider
when at least $m$ or $r$ is equal to $l$.
If $m = j =l$, by definition, the corresponding term in the summation is one.
Hence, we only have to consider the case when $m \neq r$.
By symmetry, it is equivalent to compute the partial derivative of the following
function $\frac{2}{s^2} \sum_{m=1,m\neq l}^s \sinc_\b(\w_l,\w_m)$.
Applying Lemma~\ref{lem:gradient_1}, we get the first term in \eqref{eq:gradient}.

Next, for the last term in \eqref{eq:discrepancy_gaussian},
we only have consider the term associated with one in the summation
and the term associated with $j$ in the product.
Since $\left(\frac{\sigma_j}{\sqrt{2\pi}}\right)e^{-\frac{\sigma_j^2 w^2_{lj}}{2}}
\Real\left(\erf\left(\frac{b_j}{\sigma_j\sqrt{2}} - i\frac{\sigma_j
w_{lj}}{\sqrt{2}}\right)\right)$ satisfies the formulation in Lemma~\ref{lem:gradient_2},
we can simply apply Lemma~\ref{lem:gradient_2} and get its derivative with respect to $w_{lj}$.

Equation \eqref{eq:gradient} follows by combining these terms.
\end{proof}